\documentclass[twoside]{article}

\usepackage[accepted]{aistats2021}
\usepackage[utf8]{inputenc}
\usepackage{natbib}
\usepackage{algorithmic}
\usepackage{amsmath}
\usepackage{amssymb}
\usepackage{amsthm}
\usepackage[ruled]{algorithm2e}
\usepackage{appendix}
\usepackage{xcolor}
\usepackage{enumitem}
\usepackage{color}
\usepackage{graphicx}
\usepackage{subfigure}
\usepackage{booktabs} %

\usepackage{thmtools}
\usepackage{thm-restate}

\declaretheorem[name=Theorem,numberwithin=section]{theorem}

\declaretheorem[name=Lemma, numberwithin=section]{lemma}

\declaretheorem[name=Definition, numberwithin=section]{definition}
\declaretheorem[name=Remark, numberwithin=section]{remark}

\newtheorem*{theorem*}{Theorem}

\newtheoremstyle{bftheorem}%
  {3pt}%
  {3pt}%
  {\itshape}%
  {}%
  {\bfseries}%
  {.}%
  {.5em}%
  {\thmname{#1}\thmnumber{ #2}\thmnote{ (#3)}}%

\theoremstyle{bftheorem}
\newtheorem{bftheorem}[theorem]{Theorem}%

\newcommand{\R}{\mathbb{R}}
\newcommand{\N}{\mathbb{N}}
\newcommand{\Exp}[2][]{\mathbb{E}#1[{#2}#1]}
\newcommand{\Exptilde}[2][]{\tilde{\mathbb{E}}#1[{#2}#1]}
\newcommand{\PP}{\mathbb{P}}
\newcommand{\EE}{\mathbb{E}}
\newcommand{\VV}{\mathbb{V}}
\newcommand{\Var}[2][]{\mathbb{V}#1[{#2}#1]}
\newcommand{\Vartilde}[2][]{\tilde{\mathbb{V}}#1[{#2}#1]}
\newcommand\given[1][]{\:#1\vert\:}
\newcommand{\indep}{\perp \!\!\! \perp}
\newcommand{\norm}[2][]{#1\lVert{#2}#1\rVert}
\newcommand{\CZ}{\mathcal{Z}}
\newcommand{\vopt}[1]{v({#1},\lambda)}
\newcommand{\wstoc}[1][t]{w_{#1}(\lambda)}
\newcommand{\vstoc}[2][k]{ v_{{#1}}({#2},\lambda)}
\newcommand{\LtPhi}{L_{\tilde \Phi,\lambda}}
\newcommand{\mtwo}{m_{2,\lambda}}
\newcommand{\sz}[1]{\eta_{\lambda, {#1}}}
\newcommand{\sigone}{\sigma_{\lambda, 1}}
\newcommand{\sigtwo}{\sigma_{\lambda, 2}}

\newcommand{\grad}{\nabla}
\newcommand{\jac}{\partial}

\newcommand{\fo}{E} %

\newcommand{\Low}{\mu_{1,\lambda}}
\newcommand{\Lol}{\mu_{2,\lambda}}
\newcommand{\Bo}{L_{E,\lambda}}

\newcommand{\q}{q_\lambda}

\newcommand{\rhol}{\nu_{2, \lambda}}
\newcommand{\rhow}{\nu_{1, \lambda}}
\newcommand{\LPhi}{L_{\Phi, \lambda}}

\newcommand{\rf}{\rho_\lambda}
\newcommand{\hrf}{\sigma_{\lambda}}

\newcommand{\MSE}{\text{MSE}_{\hat \nabla f}}
\newcommand{\MSEW}[1]{\text{MSE}_{w_{#1}}}

\usepackage{hyperref}
\usepackage{cleveref}

\crefname{appsec}{appendix}{appendices}

\newtheorem{assumption}{Assumption}

\begin{document}

\twocolumn[

\aistatstitle{Convergence Properties of Stochastic Hypergradients}

\aistatsauthor{Riccardo Grazzi \And Massimiliano Pontil \And Saverio Salzo}
\aistatsaddress{ Istituto Italiano di Tecnologia \\
University College London \And  Istituto Italiano di Tecnologia \\
University College London
\And Istituto Italiano di Tecnologia }
]

\begin{abstract}
Bilevel optimization problems are receiving increasing attention in machine learning as they provide a natural framework for hyperparameter optimization and meta-learning. A key step to tackle these problems 
is the efficient computation of the gradient of the upper-level objective (hypergradient).~In this work, we study stochastic approximation schemes for the hypergradient, which are 
important when the lower-level problem is empirical risk minimization on 
a large dataset. 
The method that we propose is a stochastic variant of the approximate implicit differentiation approach in (Pedregosa, 2016).
We provide
bounds for the mean square error of the hypergradient approximation, under the assumption that the lower-level problem is accessible only through a stochastic mapping which is a contraction in expectation.
In particular, our main bound is agnostic to the choice of the two stochastic solvers employed by the procedure.
We provide numerical experiments to support our theoretical analysis and to show the advantage of using stochastic hypergradients in practice.
\end{abstract}

\section{Introduction}
In this paper we study the following bilevel problem
\begin{equation}
\begin{aligned}
\label{mainprob}
&\min_{\lambda \in \Lambda} f(\lambda) := \fo(w(\lambda), \lambda)\\
&\text{\ subject~to ~}w(\lambda) = \Phi(w(\lambda), \lambda),
\end{aligned}
\end{equation}
which at the lower-level incorporates a  
(parametric) fixed-point equation.
This problem is paramount 
in many applications,  especially in machine learning and statistics, including hyperparameter optimization~\citep{maclaurin2015gradient,franceschi2017forward,liu2018darts,lorraine2019optimizing,elsken2019neural}, meta-learning~\citep{andrychowicz2016learning,finn2017model,franceschi2018bilevel}, and graph and recurrent neural networks~\citep{almeida1987learning,pineda1987generalization,scarselli2008graph}. 

In dealing with problem \eqref{mainprob},
one critical issue is to devise efficient algorithms to compute the (hyper) gradient of the function $f$, so as to allow using gradient based approaches to find a solution. 
The computation of the hypergradient via approximate implicit differentiation (AID) \citep{pedregosa2016hyperparameter} requires one to solve two subproblems: $(i)$ the lower-level problem in~\eqref{mainprob} and 
$(ii)$ a linear system which arises from the implicit expression for $\nabla f(\lambda)$.
However, especially in large scale scenarios, solving those  subproblems exactly might either be impossible or too expensive, hence, iterative approximation methods are often used.
In \citep{grazzi2020iteration},
under the assumption that, for every $\lambda \in \Lambda$, the mapping $\Phi(\cdot, \lambda)$ in \eqref{mainprob} is a contraction, a comprehensive analysis of the iteration complexity of the  hypergradient computation  for several popular 
deterministic algorithms was provided. Here, instead, we address such iteration complexity for stochastic methods.
This study is of fundamental importance since in many practical scenarios $\Phi(w, \lambda)$ is expensive to compute, e.g., when it has a sum structure with a large number of terms. 
In this situation stochastic approaches become the method of choice.
For example, in large scale hyperparameter optimization and neural architecture search \citep{maclaurin2015gradient,lorraine2019optimizing,liu2018darts}, solving the lower-level problem requires minimizing a training objective over a large dataset, which is usually done approximately through SGD and its extensions.
Our contributions can be summarized as follows.
\begin{itemize}
    \item We  devise a stochastic estimator $\hat\nabla f (\lambda)$
of the true gradient, based on the AID technique, together with an explicit bound for the related mean square error.
The bound is agnostic with respect to the stochastic methods solving the related subproblems, so that can be applied to several algorithmic solutions; see Theorem~\ref{thm:finalbound}.
\item We study the convergence of a general stochastic fixed-point iteration method which extends and improves previous analysis of SGD for strongly convex functions and can be applied to solve both subproblems associated to the AID approach. These results, which are interesting in their own right, are given in Theorems~\ref{th:innerconvergence2} and
\ref{th:innerconvergence}.

\end{itemize}

Proofs of the results presented in the paper can be found in the supplementary material.

\paragraph{Related Work}
\cite{pedregosa2016hyperparameter} introduced an efficient class of deterministic methods to compute the hypergradient through  AID together with asymptotic convergence results. \cite{rajeswaran2019meta,grazzi2020iteration} extended this analysis providing iteration complexity bounds. AID methods require to iteratively evaluate $\Phi$ and its derivatives.
In this work, we extend these methods by replacing those exact evaluations with unbiased stochastic approximations and provide iteration complexity bounds in this scenario.\\
Another class of methods (ITD) computes the hypergradient by differentiating through the inner optimization scheme  \citep{maclaurin2015gradient,franceschi2017forward,franceschi2018bilevel}. Iteration complexity results for the deterministic case are given in \citep{grazzi2020iteration}, while we are not aware of any convergence results in the stochastic setting. Here, we focus entirely on AID methods, leaving the investigation of stochastic ITD methods for future work.\\
An interesting special case of the bilevel problem~\eqref{mainprob} is when
$f(\lambda) = \min_{w} E(w,\lambda)$.
This scenario occurs for example in regularized meta-learning, where the properties of a simple stochastic hypergradient estimator have been studied extensively \citep{denevi2019learning,denevi2019online, zhou2019efficient}. In this setting, \cite{ablin2020super} analyze, among others,
implicit differentiation techniques for approximating the gradient of $f$, including stochastic approaches.
However, the proposed estimator  
assumes to solve the related linear system exactly, which is often impractical.
In this work, we focus on the more general setting of bilevel problem~\eqref{mainprob}, devising algorithmic solutions that are fully stochastic, in the sense that 
also the subproblem involving the linear system is solved by a stochastic method.\\
Finally, stochastic algorithms for hypergradient computation in bilevel optimization problems have been studied in  \citep{couellan2016convergence,ghadimi2018approximation}.
There, the authors provide convergence rates for a whole bilevel optimization procedure
using stochastic oracles both from the upper-level and the lower-level objectives.
In particular, the method used by \cite{ghadimi2018approximation} to approximate the hypergradient can be seen as a special case of our method with two particular choices of the stochastic solvers.\footnote{Specifically 
they use SGD with decreasing step sizes for the lower-level problem (which is a minimization problem) and, for the linear system, 
a stochastic routine derived from the Neumann series approximation of the matrix inverse.
}

\paragraph{Notation}
We denote by $\norm{\cdot}$ either the Euclidean norm or the spectral norm (when applied to matrices). 
The transpose
and the inverse of a given matrix $A$, is denoted by $A^\top$ and $A^{-1}$ respectively.
For a real-valued function $g\colon \R^n\times\R^m\to \R$, we denote by $\nabla_1 g(x,y) \in \R^n$ and $\nabla_2 g(x,y) \in \R^m$, the partial derivatives w.r.t. the first and second variable respectively.  For a vector-valued function $h\colon \R^n\times \R^m \to \R^k$ we denote by
$\partial_1 h(x,y) \in \R^{k\times n}$
and $\partial_2 h(x,y) \in \R^{k\times m}$ the partial Jacobians w.r.t.~the first and second variables respectively.
For a random variable $X$ we denote by $\EE[X]$ and $\VV[X]$ its expectation and variance respectively.
Finally, given two random variables $X$ and $Y$,
the conditional variance of $X$ given $Y$ is $\Var{X \given Y} := \Exp{\norm{X -\Exp{X \given Y}}^2 \given Y}$.
In the following, for the reader's convenience, we provide a list of the main functions and constants used in the subsequent analysis.
\begin{table}[ht]
\centering
\caption{Table of Notation}\label{tb:notation}
\hspace{.00001truecm}
\small
\begin{tabular}{@{}cll@{}}
\toprule
\multicolumn{1}{l}{\textbf{Symbol(s)}} & \textbf{Description}                                                                  &  \\ \midrule
$\fo$                               & Upper-level objective                                                                 &  \\
$\Phi$                              & Fixed-point map                                                                       &  \\
$\hat \Phi$                         & Unbiased estimator of $\Phi$                                                          &  \\
$\hat\ell$                              & Estimator of the lower-level objective                                                                 &  \\
$\q$                                & Contraction constant of $\Phi(\cdot, \lambda)$                                        &  \\
$\Bo$                               & Lipschitz constant of $\fo(\cdot, \lambda)$                                           &  \\
$\rhow, \rhol$                      & Lipschitz const. of $\jac_1 \Phi (\cdot, \lambda), \jac_2 \Phi (\cdot, \lambda)$   &  \\
$\Low, \Lol$                        & Lipschitz const. of $\nabla_1 \fo (\cdot, \lambda), \nabla_2 \fo (\cdot, \lambda)$ &  \\
$\LtPhi$                            & Lipschitz const. of $\hat \Phi(\cdot, \lambda, \zeta)$                              &  \\
$\mtwo$                             & Bound on the variance of $\jac_2 \hat\Phi(w, \lambda, \zeta)$                         &  \\
$\rf(t),\hrf(k)$                    & \begin{tabular}[c]{@{}l@{}}Convergence rates for 
the two 
\\subproblems: $t, k$ are the number of
\\iterations of the  solvers. \end{tabular} &   \\
\bottomrule
\end{tabular}
\end{table}

\section{Stochastic Hypergradient Approximation}
\label{se:SID}

In this section we describe a general method for generating a stochastic approximation of the (hyper) gradient of $f$ in \eqref{mainprob}.
We assume that $\Phi$ is defined by an expectation of a given function $\hat{\Phi}$, that is, we consider bilevel problems
of type \eqref{mainprob} with
\begin{equation}
\label{eq:Phi_stoch}
\Phi(w,\lambda)
= \EE[\hat{\Phi}(w(\lambda), \lambda, \zeta)],
\end{equation}
where $\zeta$ is a random variable taking values in a suitable measurable space. A special case of \eqref{mainprob}-\eqref{eq:Phi_stoch}, which occurs often in machine learning, is 
\begin{equation}
\begin{aligned}
\label{mainprobopt}
&\min_{\lambda \in \Lambda} f(\lambda) := \fo(w(\lambda), \lambda)\\
&\text{\ subject~to ~}w(\lambda) = \mathrm{argmin}_{w} \EE[\hat{\ell}(w,\lambda,\zeta)],
\end{aligned}
\end{equation}
where $w\mapsto \EE[\hat{\ell}(w,\lambda,\zeta)]$ is strongly convex and Lipschitz smooth, for every $\lambda \in \Lambda$. Indeed, \eqref{mainprobopt} follows from \eqref{mainprob} and \eqref{eq:Phi_stoch} by choosing  $\hat{\Phi}(w,\lambda,\zeta) = w - \alpha_\lambda \nabla \hat{\ell}(w,\lambda,\zeta)$, for any $\alpha_\lambda>0$.

In the rest of the paper we will consider the following assumptions\footnote{Similar assumptions, except for \Cref{ass:aid}\ref{ass:lipE}, are also considered in \cite{grazzi2020iteration}.}.%
\begin{assumption}
\label{ass:aid} 
The set $\Lambda \subseteq \R^m$ is
closed and convex and
the mappings $\Phi\colon \R^d\times \R^m\to \R^d$
and $E\colon \R^d\times\R^m \to \R$
are differentiable. For every $\lambda \in \Lambda$, we assume
\begin{enumerate}[label={\rm (\roman*)}]
\item\label{ass:contraction} $\Phi(\cdot,\lambda)$ is a  contraction,  i.e., $\norm{\jac_1 \Phi(w, \lambda)} \leq \q$ for some $\q<1$ and for all $w \in \R^d$.
\item
$\jac_1\Phi (\cdot, \lambda)$ and $\jac_2\Phi (\cdot, \lambda)$ are Lipschitz 
continuous with constants $\rhow$ and $\rhol$ respectively.
\item
$\grad_1 \fo(\cdot, \lambda)$ and $\grad_2 \fo(\cdot, \lambda)$
are Lipschitz continuous with constants $\Low$ and $\Lol$ respectively.
\item\label{ass:lipE} $\fo(\cdot,\!\lambda)$ is Lipschitz 
continuous with constant $\Bo$.
\end{enumerate}
\end{assumption}

Under Assumption~\ref{ass:aid},
 $\Phi(\cdot, \lambda)$ has a unique fixed point $w(\lambda)$ and the hypergradient is given by
\begin{equation}
\begin{aligned}
\label{eq:gradfv2}
    \grad f (\lambda) &= \grad_2 \fo(w(\lambda), \lambda) \\
    & \quad + \jac_2 \Phi (w(\lambda), \lambda)^\top \vopt{w(\lambda)},
\end{aligned}
\end{equation}
where, 
\begin{equation}
\label{eq:linsys}
    \vopt{w} : = \left(I - \jac_1 \Phi (w, \lambda)^\top \right)^{-1} \grad_1 \fo(w,\lambda).
\end{equation}

This formula follows by differentiating the fixed point conditions for the lower-level problem 
and noting that, 
because of \Cref{ass:aid}\ref{ass:contraction},
$I - \jac_1 \Phi(w,\lambda)^\top$ is invertible (see Lemma~\ref{lm:matrixinverse}).

We also consider the following properties for $\hat \Phi$.
\begin{assumption}\label{ass:phiestimator}
The random variable $\zeta$ takes values in measurable space $\CZ$ 
and $\hat \Phi : \R^d \times \R^m \times \CZ \mapsto \R^d $ 
is a measurable function, differentiable w.r.t. the first two
arguments, and such that, for all 
$w \in \R^d$ and $\lambda \in \Lambda$
\begin{enumerate}[label={\rm (\roman*)}]
    \item\label{eq:expjacexp} 
    $\Exp{\hat \Phi(w, \lambda, \zeta)} {=} \Phi(w,\lambda)$~and~ 
    $\Exp{\norm{\hat \Phi(w, \lambda, \zeta)}^2}{<}\infty$.
    \item For $j \in \{1,2\}$, $\Exp{\jac_j \hat \Phi(w, \lambda, \zeta)} =  \jac_j \Exp{\hat \Phi(w, \lambda, \zeta)}$ and $\Exp{\norm{\jac_j \hat \Phi(w, \lambda, \zeta)}^2}<+\infty$.
    \item\label{eq:expjacexp_iii} For every $z \in \CZ$, $\norm{\jac_1 \hat \Phi(w, \lambda, z)} \leq  \LtPhi$ for some constant $\LtPhi \geq 0$ (which does not depend on $w$).
    \item\label{eq:expjacexp_iv} $\Var{ \jac_2 \hat\Phi(w,\lambda, \zeta)}
    \leq \mtwo$, for some $\mtwo \geq 0$ (which does not depend on $w$).
\end{enumerate}
\end{assumption}

\begin{algorithm}[t]
\caption{Stochastic Implicit Differentiation (SID)}
\label{algo1}
\begin{enumerate}
\item  Let $t \in \N$ and compute $\wstoc$ by $t$ steps of a stochastic 
algorithm that approximates $w(\lambda)$.
\item Let $k \in \N$ and Compute $\vstoc{\wstoc}$ by $k$ steps of a stochastic solver for
the linear system
\begin{equation}
\label{eq:linsystem}
    (I -\jac_1 \Phi(\wstoc, \lambda)^\top ) v = \grad_1 \fo (\wstoc, \lambda).
        \vspace{-.25truecm}
\end{equation}
\item Compute the approximate hypergradient as
\begin{align*}
  \hat \grad f(\lambda) : =& \grad_2 \fo(\wstoc, \lambda) \\
   & + \jac_2 \hat \Phi(\wstoc, \lambda, \zeta)^\top \vstoc{\wstoc}.
      \vspace{-.25truecm}
\end{align*}
\end{enumerate}
\end{algorithm}

Motivated by \eqref{eq:gradfv2}-\eqref{eq:linsys}, we consider to have at our disposal two stochastic solvers which exploit $\hat \Phi$: one for the lower-level problem in \eqref{mainprob}
which generates 
a stochastic process $\wstoc$ estimating $w(\lambda)$ and another for the linear system
\begin{equation}
\label{eq:20200601a}
    (I -\jac_1 \Phi(w, \lambda)^\top ) v = \grad_1 \fo (w, \lambda),\quad\text{with } w \in \R^d,
\end{equation}
generating a stochastic process $\vstoc{w}$ approximating
the solution $\vopt{w}$ of \eqref{eq:20200601a}.
Then, the stochastic approximation to the hypergradient is defined as
\begin{equation}\label{eq:gradestim}
\begin{aligned}
    \hat \grad f(\lambda) : = & \grad_2 \fo(\wstoc, \lambda) \\ & + \jac_2 \hat \Phi(\wstoc, \lambda, \zeta)^\top \vstoc{\wstoc}.
\end{aligned}
\end{equation}
We also suppose that, for every $w \in \R^d$, $\wstoc$, $\vstoc{w}$, and $\zeta$ are mutually independent.
The procedure, which we call SID, is summarized in Algorithm~\ref{algo1}.
In Section~\ref{sec:lower-level} we will give a way to generate the stochastic processes $(\wstoc)_{t \in \N}$ and $(\vstoc[k]{w}{\lambda})_{k \in \N}$.

\section{Mean Square Error Bound for SID}\label{se:MSE}

In this section, we derive a bound for the mean square error of the SID estimator, i.e., 
\begin{equation}
    \label{eq:MSE}
    \MSE := \Exp{\norm{\hat \grad f(\lambda) - \grad f(\lambda)}^2}.
\end{equation}
 To that purpose, we require the stochastic procedures at point 1 and 2 of \Cref{algo1} to have non-asymptotic convergence rates 
 in mean square.
This is the content of the following assumption.

\begin{assumption}\label{ass:innerbackrates}
For every $\lambda \in \Lambda$,
 $t,k \geq 1$ and $w \in \R^d$, the random variables  $\vstoc{w}$, $\wstoc$ and $\zeta$ are mutually independent and 
\begin{align*}
    \Exp{\norm{\wstoc -w(\lambda)}^2} &\leq \rf(t) \\ 
    \Exp{\norm{\vstoc{w} - \vopt{w}}^2} &\leq \hrf(k),
\end{align*}
where $\rf:\N \mapsto \R_{+}$ and $\hrf:\N \mapsto \R_{+}$.
\end{assumption}

This assumption is often satisfied in applications, e.g., 
in problems of type \eqref{mainprobopt}, when the lower-level objective is strongly convex and Lipschitz smooth. In Section~\ref{sec:fixed-point} we describe a general stochastic fixed-point method from which, in Section~\ref{sec:lower-level},
we will derive a stochastic implicit differentiation method
featuring the rates required in Assumption~\ref{ass:innerbackrates}.

In order to analyze the quantity in \eqref{eq:MSE}, 
we start with the standard bias-variance decomposition
 (see Lemma~\ref{lem:varprop}) as follows
\begin{equation}
\begin{aligned}\label{eq:biasvar}
    \MSE =& \underbrace{\norm{\Exp{\hat \grad f(\lambda)} - \grad f(\lambda)}^2}_{\text{bias}} +  \underbrace{\Var{\hat \grad f(\lambda)}}_{\text{variance}}.
\end{aligned} 
\end{equation}
Then, using the law of total variance (see \Cref{lm:totvariance}), we write the mean square error as below
\begin{equation}\label{eq:20200601h}
\begin{aligned}
    &\MSE %
    = \underbrace{\norm{\Exp{\hat \grad f(\lambda)} - \grad f(\lambda)}^2}_{\text{bias}} \\
    &\ \ +\underbrace{\Exp{\Var{\hat \grad f(\lambda) \given \wstoc}} + \Var{\Exp{\hat \grad f(\lambda) \given \wstoc}}}_{\text{variance}}.
\end{aligned}
\end{equation}
In the following we will bound each term on the right-hand side of \eqref{eq:20200601h} individually.
The next result serves to control the bias term.

\begin{theorem}
\label{th:boundbias}
Suppose that Assumptions~\ref{ass:aid},\ref{ass:phiestimator}, and \ref{ass:innerbackrates}
are satisfied. Let $\lambda \in \Lambda$, $t,k \in \N$ and set
\begin{equation*}
\begin{aligned}
\hat \Delta_w := \norm{\wstoc - w(\lambda)},\ 
\LPhi := \norm{\jac_2 \Phi(w(\lambda), \lambda)}, \\[1ex] 
    c_{1,\lambda} = \Lol + \frac{\Low\LPhi + \rhol\Bo}{1-\q} + \frac{\rhow\Bo\LPhi}{(1-\q)^2}.
\end{aligned}
\end{equation*}
Then the following hold.
\begin{enumerate}[label={\rm(\roman*)}]
\item\label{th:boundbias_i} 
$\norm[\big]{ \Exp{\hat \grad f(\lambda) \given \wstoc}   - \grad f(\lambda)}$ \\[1ex]
$\leq c_{1,\lambda}\hat \Delta_w
+ \LPhi \sqrt{\hrf(k)} + \rhol\hat \Delta_w \sqrt{\hrf(k)}$.
\item\label{th:boundbias_ii} 
$\norm{\EE[\hat \grad f(\lambda)] - \grad f(\lambda)}$\\[1ex]
$\leq c_{1,\lambda} \sqrt{\rf(t)} + \LPhi  \sqrt{\hrf(k)} + \rhol \sqrt{\rf(t)}\sqrt{\hrf(k)}.$
\end{enumerate}
\end{theorem}
The following two theorems provide bounds for the two components of the variance in \eqref{eq:20200601h}.

\begin{theorem}
\label{th:varboundone}
Suppose that Assumptions~\ref{ass:aid},\ref{ass:phiestimator}, and \ref{ass:innerbackrates}
are satisfied. Let $\lambda \in \Lambda$, $t,k \in \N$ and set $\LPhi := \norm{\jac_2 \Phi(w(\lambda), \lambda)}$. 
Then 
\begin{equation}
\begin{aligned}
\Exp[\big]{\Var{\hat \grad f(\lambda) \!\given\! \wstoc}} &\leq 2\frac{\mtwo\Bo^2}{(1-\q)^2} \\[1ex] 
&\quad + 2(\LPhi^2 + \mtwo)\hrf(k) \\[1ex]
&\quad + 2\rhol^2 \rf(t)\hrf(k).
\end{aligned}
\end{equation}
\end{theorem}

\begin{theorem}
\label{th:varboundtwo}
Suppose that Assumptions~\ref{ass:aid},\ref{ass:phiestimator}, and \ref{ass:innerbackrates}
are satisfied. Let $\lambda \in \Lambda$, and $t,k \in \N$. Then 
\begin{equation*}
\begin{aligned}
    \Var{\Exp{\hat \grad f(\lambda) \given \wstoc}} \leq 
    3 \big( & c_{1,\lambda}^2\rf(t) +\LPhi^2\hrf(k) \\[1ex]
    & + \rhol^2 \rf(t)\hrf(k) \big),
\end{aligned}
\end{equation*}
where $c_{1,\lambda}$ and $\LPhi$ are defined as in \Cref{th:boundbias}.
\end{theorem}

Finally, combining the above three results, we give the promised bound on the mean square error for the estimator of the hypergradient.

\begin{bftheorem}[MSE bound for SID]
\label{thm:finalbound}
Suppose that Assumptions~\ref{ass:aid},\ref{ass:phiestimator}, and \ref{ass:innerbackrates}
are satisfied. Let $\lambda \in \Lambda$, and $t,k \in \N$. Then
\begin{equation}
\label{eq:finalbound}
\begin{aligned}
    \MSE \leq &\ 2\frac{\mtwo \Bo^2}{(1-\q)^2} 
    + 6 c^2_{1,\lambda}\rf(t) \\ &+ 2(4\LPhi^2 + \mtwo) \hrf(k)
    \\&+ 8\rhol^2 \rf(t)\hrf(k).
\end{aligned}
\end{equation}
where $c_{1,\lambda}$ is defined as in Theorem~\ref{th:boundbias}.
In particular, if $\lim_{t\to \infty}\rf(t) = 0$ and $\lim_{k\to \infty}\hrf(k) = 0$, then
\begin{equation}\label{eq:asmsebound}
    \lim_{t,k \to \infty}  \MSE \leq  2\frac{\mtwo \Bo^2}{(1-\q)^2}. 
\end{equation}
\end{bftheorem}
\begin{proof}
The statement follows from the decomposition 
\eqref{eq:20200601h} and Theorems~\ref{th:boundbias}, \ref{th:varboundone}, and \ref{th:varboundtwo}. 
\end{proof}

In the following we make few comments related to the above results.
First, it follows from (ii) in Theorem~\ref{th:boundbias} that, if $\lim_{t\to \infty}\rf(t) = 0$ and $\lim_{k\to \infty}\hrf(k) = 0$, 
the estimator  $\hat{\nabla} f(\lambda)$
is asymptotically unbiased as $t,k \to +\infty$.
Next, the bound \eqref{eq:finalbound} in Theorem~\ref{thm:finalbound}
provides the iteration complexity of the SID method  (Algorithm~\ref{algo1}). This result is a stochastic version of what was obtained by \cite{grazzi2020iteration} concerning 
approximate implicit differentiation methods.
Finally, 
note that it follows from \eqref{eq:asmsebound} that the mean square error 
cannot be made arbitrarily small unless the variance
term $\Var{ \jac_2 \hat\Phi(w,\lambda, \zeta)}$ (controlled by $m_{2,\lambda}$) is zero. 
This may seem a limitation of the method. However, since SID uses $\jac_2 \hat\Phi(w,\lambda, \zeta)$ only once at the end of the procedure, one could modify the algorithm 
by sampling $\zeta$ several times so to reduce the variance of $\jac_2 \hat\Phi(w,\lambda, \zeta)$ or, when possible, even compute $\jac_2 \Phi(w, \lambda)$ exactly, with little increase on the overall cost.
Additionally, we stress that in several applications that variance term is zero.
Indeed, this occurs each time $\hat{\Phi}$ is of the form
\begin{equation}
    \hat{\Phi}(w,\lambda,\zeta) = 
    \hat{\Phi}_1(w,\zeta) + \hat{\Phi}_2(w,\lambda),
\end{equation}
meaning that, $\hat{\Phi}$ depends on
the random variable $\zeta$
and on the hyperparameter $\lambda$  
in a separate manner.
For instance, this is the case when we want to optimize the regularization hyperparameters
in regularized empirical risk minimization problems,
where usually, the random variable $\zeta$ affects only the data term.

\section{Stochastic fixed-point iterations}
\label{sec:fixed-point}
\newcommand{\T}{T}
\newcommand{\x}{w}
\newcommand{\sW}{\mathfrak{W}}

In this section we address the convergence of stochastic fixed-point iteration methods which can be applied in a similar manner to solve both subproblems in \Cref{algo1} (see \Cref{sec:lower-level}).
We consider the general situation of computing the fixed point of a contraction mapping which is accessible only through a stochastic oracle. 
The results are inspired by the analysis of the SGD algorithm for strongly convex and Lipschitz smooth functions
given in \citep{bottou2018optimization}, but
extended to our more general setting. Indeed, by a more accurate computation of the contraction constant of the gradient descent mapping, we are able to improve the convergence rates and increase the stepsizes given in the above cited paper. See Corollary~\ref{th:sgd}
and the subsequent remark. We stress that the significance of the results presented in this section 
goes beyond the bilevel setting \eqref{mainprob}-\eqref{eq:Phi_stoch} 
and may be of interest per se. 

We start with the  assumption below.

\begin{assumption}\label{ass:inner}
Let $\zeta$ be a random variable with values in a measurable space $\CZ$. Let $\T:\R^d \mapsto \R^d$ and $\hat\T:\R^d \times \CZ \mapsto \R^d$ be such that
\begin{enumerate}[label={\rm (\roman*)}]
\item\label{ass:cont} 
$\forall \x_1,\x_2 \in \R^d$,
$\norm{\T(\x_1) - \T(\x_2)} \leq q\norm{\x_1 -\x_2}$,\\[1ex] with $q<1$.
\item\label{ass:tildeexp}
$\forall \x \in \R^d$,
$\Exp{\hat\T(\x,\zeta)} = \T(\x)$
\item\label{ass:tildevar}
$\forall \x \in \R^d$,
$\Var{\hat\T(\x,\zeta)} \leq \sigma_1 + \sigma_2 \norm{\T(\x) - \x}^2$.
\end{enumerate}
\end{assumption}

The above assumptions are in line with those made by \cite{bottou2018optimization} for the case of stochastic minimization of a strongly convex and Lipschitz smooth function.

Since $\T$ is a contraction, there exists a unique $\x^* \in \R^d$ such that
\begin{equation}
    \x^* = \T(\x^*).
\end{equation}
We consider the following random process which corresponds to a stochastic version of the 
Krasnoselskii-Mann iteration for contractive operators. Let $(\zeta_t)_{t \in \N}$ be a sequence of independent copies of $\zeta$. Then,
starting from $\x_0 \in \R^d$ we set
\begin{equation}\label{eq:inneralgo}
(\forall\, t \in \N)\quad    \x_{t+1} = \x_t + \eta_t (\hat \T(\x_t, \zeta_t) - \x_t).
\end{equation}

The following two results provide non-asymptotic convergence rates for the procedure \eqref{eq:inneralgo} for two different strategies about the step-sizes $\eta_t$.
\begin{bftheorem}[Constant step-size]
\label{th:innerconvergence2} 
Let Assumption~\ref{ass:inner} hold and 
suppose that $\eta_t = \eta \in \R_{++}$, for every $t \in \N$, and that
\begin{equation*}
    \eta \leq \frac{1}{1+\sigma_2}.
\end{equation*}
Let $(\x_t)_{t \in \N}$ be generated according to algorithm \eqref{eq:inneralgo} and set
 $\MSEW{t} : = \Exp{\norm{\x_t - \x^*}^2}$. Then, for all $t \in \N$,
\begin{equation}\label{eq:rateconstant}
\begin{aligned}
\MSEW{t} &\leq (1-\eta(1-q^2))^t\bigg(\MSEW{0} -\frac{\eta \sigma_1}{1-q^2} \bigg) \\ &\qquad + \frac{\eta \sigma_1}{1-q^2}.
\end{aligned}
\end{equation}
In particular, $\lim_{t \to \infty} \MSEW{t} \leq \eta \sigma_1/(1 - q^2)$.
\end{bftheorem}

\begin{bftheorem}[Decreasing step-sizes]
\label{th:innerconvergence}
Let Assumption~\ref{ass:inner} hold and 
suppose that for every $t \in \N$
\begin{equation}
\label{eq:20201013b}
    \eta_t \leq \frac{1}{1+\sigma_2}, \quad
    \sum_{t=1}^\infty \eta_t = \infty, \quad 
    \sum_{t=1}^\infty \eta^2_t < \infty.
\end{equation}
Let $(\x_t)_{t \in \N}$ be generated according to Algorithm~\eqref{eq:inneralgo}.
Then 
\begin{equation*}
    \x_t \to \x^* \qquad \PP\text{-a.s.}
\end{equation*}
Moreover, if $\eta_t = \beta/(\gamma  + t)$, 
with $\beta > 1/(1-q^2)$ and  $\gamma \geq \beta(1 + \sigma_2)$,
then we have
\begin{equation}
\label{eq:ratediminish}
    \Exp{\norm{\x_t - \x^*}^2} 
    \leq \frac{c}{\gamma + t},
\end{equation}
where 
\begin{equation*}
    c : = \max \left\{\gamma \Exp{\norm{\x_0 - \x^*}^2}, \frac{\beta^2\sigma_1}{\beta(1-q^2) -1} \right\}.
\end{equation*}
\end{bftheorem}

We will now comment on the choice of the stepsizes in algorithm~\eqref{eq:inneralgo}.
Theorem~\ref{th:innerconvergence2} and \ref{th:innerconvergence}
suggest that it may be convenient to start the algorithm with a constant stepsize.
Then, once reached a mean square error approximately less than $\eta \sigma_1/(1-q^2)$, the stepsizes should change
regime and start decreasing according to Theorem~\ref{th:innerconvergence}. More precisely, 
in the first phase it is recommended to set $\eta=1/(1+\sigma_2)$ in order to
maximize the stepsize. Then, 
the second phase should be initialized with
$w_0$ such that $\MSEW{0} \leq \sigma_1/[(1+\sigma_2)(1-q^2)]$
and $\gamma = \beta(1+\sigma_2)$ so that
\begin{align*}
    \gamma \Exp{\norm{\x_0 - \x^*}^2} 
    &\leq \frac{\beta^2\sigma_1}{\beta(1-q^2) -1}.
\end{align*}
In this situation, $c$ will be dominated by its second term, which is minimized when $\beta = 2/(1-q^2)$. Similar suggestions are made in \citep{bottou2018optimization}. 

In the following, partly inspired by  the analysis of \cite{nguyen2019new}, we show that, with an additional Lipschitz assumption on $\hat \T,$ which is commonly verified in practice, Assumption~\ref{ass:inner}\ref{ass:tildevar} on the variance of the estimator is satisfied. The following \Cref{ass:tildephilip} is an extension of Assumption 2 in \citep{nguyen2019new}.

\newcommand{\LtT}{L_{\hat \T}}
\begin{assumption}\label{ass:tildephilip}
There exists $\LtT \geq 0 $  such that, for every $\x_1,\x_2 \in \R^d$ and for every $z \in \CZ$
\begin{equation*}
\norm{\hat \T(\x_1, z) - \hat \T(\x_2, z)} \leq \LtT\norm{\x_1 - \x_2}. 
\end{equation*}
\end{assumption}

\begin{restatable}{theorem}{liptovar}\label{th:liptovar}
Suppose that
 \Cref{ass:tildephilip} and \Cref{ass:inner}\ref{ass:cont}\ref{ass:tildeexp}  hold. Then \Cref{ass:inner}\ref{ass:tildevar} holds. In particular, for every $w \in \R^d$,
\begin{equation*}
 \Var{\hat \T(\x, \zeta)} \leq \underbrace{2\Var{\hat \T(\x^*, \zeta)}}_{\sigma_1} + \underbrace{2\frac{\LtT^2 + q^2}{(1-q)^2}}_{\sigma_2}\norm{\T(\x)-\x}^2.
\end{equation*}
\end{restatable}

We now discuss the popular case of SGD and make a comparison with the related results by \cite{bottou2018optimization}. We assume that $\hat\T(\x, \zeta) = \x - \alpha \grad \hat{\ell}(w, \zeta)$, for a suitable $\alpha>0$. With this choice, algorithm \eqref{eq:inneralgo} becomes
\begin{equation}\label{eq:sgd}
(\forall\, t \in \N)\quad    \x_{t+1} = \x_t - \eta_t \alpha \grad_1 \ell (\x_t, \zeta_t),
\end{equation}
which is exactly stochastic gradient descent.
We have the following assumption on $\hat{\ell}$.

\begin{assumption}\label{ass:sgd}
$\hat{\ell}: \R^d \times \CZ \to \R $ is twice continuously differentiable w.r.t. the first variable. Let $\ell(\x) := \Exp{ \hat{\ell}(w, \zeta)}$.
\begin{enumerate}[label={\rm (\roman*)}]
\item\label{ass:lipsc} 
$\ell(\x)$ is $\tau$ strongly convex and $L$-smooth
\item\label{ass:tildeLvar}
$\forall \x \in \R^d$,
$\Var{ \grad \hat{\ell}(\x,\zeta)} \leq
\sigma^\prime_1 + \sigma^\prime_2 \norm{\grad \ell(\x)}^2$.
\end{enumerate}
\end{assumption}

\begin{restatable}{corollary}{sgd}
\label{th:sgd} Let Assumption~\ref{ass:sgd} hold and
let $(\x_t)_{t \in \N}$ be generated according to algorithm~\eqref{eq:sgd} with $\eta_t=\eta \leq 1/(1+\sigma^\prime_2)$. Then
\begin{equation}
\label{eq:rateconstantsgd}
\Exp{\norm{\x_t - \x^*}^2} \leq r_1^t\big(\Exp{\norm{\x_0 - \x^*}^2} - r_2 \big) + r_2,
\end{equation}
where 
\begin{align*}
    r_1 &: = \begin{cases}
    1 - \dfrac{\eta \tau}{L}\Big( 2 - \dfrac{\tau}{L} \Big) &\text{if } \alpha=1/L\\[1.5ex]
    1 - 4\dfrac{\eta\tau L}{(L+\tau)^2} &\text{if } \alpha = 2/(L+\tau).
    \end{cases}\\[1ex]
    r_2 &: = \begin{cases}
    = \dfrac{\eta \sigma^\prime_1}{\tau (2 L- \tau)} &\text{if } \alpha = 1/L\\[2ex]
     \dfrac{\eta \sigma^\prime_1}{\tau L} &\text{if } \alpha = 2/(L+\tau).
     \end{cases}
\end{align*}
Moreover, let $\eta_t = \beta/(\gamma  + t)$, 
where
\begin{equation}
\beta >
\begin{cases}
\dfrac{L^2}{\tau (2 L - \tau)} &\text{if } \alpha = 1/L\\[2ex]
\dfrac{(L+\tau)^2}{4 \tau L} &\text{if } \alpha = 2/(L+\tau)
\end{cases}
\end{equation}
and  $\gamma \geq \beta(1 + \sigma^\prime_2)$.
Then, for all $t \in \N$, we have
\begin{equation}\label{eq:ratediminishsgd}
 \Exp{\norm{\x_t - \x^*}^2} 
    \leq \frac{\max \{\gamma \Exp{\norm{\x_0 - \x^*}^2}, r_3 \}}{\gamma + t},
\end{equation}
where 
\begin{equation*}
    r_3:= \begin{cases}
    \dfrac{\beta^2\sigma^\prime_1}{\beta \tau(2 L-\tau) -L^2} &\text{if } \alpha = 1/L\\[2ex]
    \dfrac{4\beta^2\sigma^\prime_1}{4\beta \tau L -(L+\tau)^2} &\text{if } \alpha = 2/(L+\tau).
    \end{cases}
\end{equation*}
\end{restatable}

\begin{remark}
\label{rmk:20201013a}
In \citep{bottou2018optimization},
under Assumption~\ref{ass:sgd} a rate equal to \eqref{eq:rateconstantsgd} is obtained, but with $\alpha = 1/L$
and
\begin{equation}
\label{eq:20201008a}
r_1 = 1-\eta \frac{\tau}{L}
\quad\text{and}\quad 
r_2 = \frac{\eta \sigma_1^\prime}{2 \tau}.
\end{equation}
We see then, that Corollary~\ref{th:sgd} provides better rates. Also, our analysis allows choosing the larger (and optimal) stepsize $2/(L+\tau)$.
\end{remark}

\begin{remark}\label{rem:sig2}
In Assumption~\ref{ass:sgd}, suppose that $\zeta$ takes values in $\CZ=\{1,\dots, n\}$
with uniform distribution and that for every $i\in \{1,\dots, n\}$, $\hat{\ell}(\cdot, i)$ is strongly convex with modulus $\tau$.
This is, for instance, the case 
of the regularized empirical risk functional,
\begin{equation}
    \hat{\ell}(w, i) = \psi(y_i  w^\top x_i) +  \frac{\tau}{2}  \norm{w}^2,
\end{equation}
where $(x_i,y_i)_{1 \leq i \leq n} \in (\R^d\times\{1,2\})^n$
is the training set. Then,
if the loss function $\psi$ is Lipschitz continuous, as is the case, e.g., of the logistic loss, we have
\begin{align}
\nonumber\VV[\nabla \hat{\ell}(w,i)] 
    &= \VV_{i\sim U[\mathcal{Z}]}[ \psi^\prime(y_i  w^\top x_i) y_i x_i]\\
    & \leq \mathrm{Lip}(\psi)^2 \EE_{i\sim U[\mathcal{Z}]}[\norm{x_i}^2],
\end{align}
so that Assumption~\ref{ass:sgd}\ref{ass:tildeLvar} is satisfied with $\sigma_2^\prime = 0$.
\end{remark}

\section{Solving the Subproblems in SID}
\label{sec:lower-level}

We are now ready to show how to generate the sequences $\wstoc$ and $\vstoc{\wstoc}$ required by Algorithm~\ref{algo1}.
Let $\zeta^\prime$ be a random variable with values in $\CZ$ 
satisfying Assumption~\ref{ass:phiestimator}\ref{eq:expjacexp}.
Let $(\zeta_i)_{i\in \N}$ and $(\hat \zeta_i)_{i\in \N}$ be independent copies of $\zeta^\prime$ and independent from each other, and let $(\sz{i})_{i\in \N}$ be a sequence of stepsizes such that 
$\sum_{i=0}^\infty \sz{i} = +\infty$ and $\sum_{i=0}^\infty \sz{i}^2 < +\infty$. 
For every $w \in \R^d$ we let $\wstoc[0] = \vstoc[0]{w} = 0$, and, for $k,t \in \N$,
\begin{equation}\label{eq:inneralg}
    \wstoc[t+1] : = \wstoc[t] + \sz{t} (\hat \Phi(\wstoc[t], \lambda, \zeta_{t}) - \wstoc[t])      
\end{equation}
and
\begin{equation}\label{eq:backalg}
\begin{aligned}
    \vstoc[k+1]{w} : = &\vstoc[k]{w} \\ & + \sz{k} (\hat \Psi_w(\vstoc[k]{w}, \lambda, \hat \zeta_{k}) - \vstoc[k]{w}), 
\end{aligned}
\end{equation}

where $\hat \Psi_w(v, \lambda, z) : = \jac_1 \hat \Phi (w, \lambda, z)^\top v + \grad_1 \fo (w, \lambda)$.

We note that
if the Jacobian-vector product above is computed using reverse mode automatic differentiation, the costs of evaluating $\hat \Psi_w$ and $\hat \Phi$ are of the same order of magnitude.
Furthermore, thanks to the definition of $\hat \Psi$, we can solve both subproblems in \Cref{algo1} using the procedure described in \Cref{sec:fixed-point}. In particular, if  we set $\sz{t} =\sz{k}$, we can obtain similar convergence guarantees for both \eqref{eq:inneralg} and \eqref{eq:backalg}. The case of decreasing step sizes is treated in the following result, which is a direct consequence of Theorem~\ref{th:innerconvergence}.

\begin{restatable}{theorem}{firstconv}
\label{th:firstconv}
Let \Cref{ass:aid}\ref{ass:contraction} and \ref{ass:phiestimator} hold. Let $\lambda \in \Lambda$ and let $\wstoc$ and $\vstoc{w}$ be defined as in \eqref{eq:inneralg} and \eqref{eq:backalg}. 
Then, for every  $w \in \R^d$, we have
\begin{equation*}
        \lim_{t \to \infty} \wstoc = w(\lambda), \quad
        \lim_{k \to \infty}\vstoc{w} = \vopt{w} \quad \PP\text{-a.s.}
\end{equation*}
Moreover, let $\sigtwo := 2(\LtPhi^2 + \q^2)/(1-\q)^2$ and $\sz{i}:= \beta_\lambda/(\gamma_\lambda + i)$ with 
$\beta_\lambda > 1/(1-\q^2)$ and $\gamma_\lambda \geq \beta_\lambda(1+\sigtwo)$.
Then for every $w \in \R^d$
\begin{align}
    \Exp{\norm{\wstoc - w(\lambda)}^2} &\leq \frac{ d_{w,\lambda}}{\gamma_\lambda + t} \\
    \Exp{\norm{\vstoc{w} - \vopt{w}}^2} &\leq \frac{d_{v,\lambda}}{\gamma_\lambda + k} 
\end{align}
where 
\begin{align*}
    d_{w,\lambda} &:= \max \left\{\gamma_\lambda \norm{w(\lambda)}^2, \frac{\beta_\lambda^2\sigone}{\beta_\lambda(1-\q^2) - 1}\right\}, \\
    d_{v,\lambda} &:= \frac{\norm{\nabla_1 \fo (w,\lambda)}^2}{(1-\q)^2}\max \left\{\gamma_\lambda , \frac{2\beta_\lambda^2\LtPhi^2}{\beta_\lambda(1-\q^2) - 1}\right\}\\
    \sigone &:= 2\Var{\hat \Phi(w(\lambda), \lambda, \zeta)}.
\end{align*}

\end{restatable}
In a similar manner, using Theorem~\ref{th:innerconvergence2}, one can have rates of convergence also using a constant stepsize, although in that case, we do not have asymptotic convergence of the iterates.

\begin{remark}
\label{rmk:20201014a}
For the setting considered in \Cref{th:firstconv}
the bound given in Theorem~\ref{thm:finalbound} yields
\begin{equation}
        \MSE \leq 2\frac{\mtwo \Bo^2}{(1-\q)^2} + O\left(\frac{1}{\gamma_\lambda + t} + \frac{1}{\gamma_\lambda + k}\right).
\end{equation}
where $\MSE =  \Exp[\big]{\norm{\hat \grad f(\lambda) - \grad f(\lambda)}^2}$
\end{remark}

Crucially, typical bilevel problems 
in machine learning come in the form of \eqref{mainprobopt},
where the lower-level objective $\ell (w, \lambda) := \EE[\hat{\ell}(w,\lambda,\zeta)]$ is Lipschitz smooth and strongly convex w.r.t.~$w$.
In this scenario, there is a vast amount of stochastic methods in literature (see e.g. \cite{bottou2018optimization} for a survey) achieving convergence rates in expectations of the kind  provided in Theorem~\ref{th:firstconv} or even better.
For example, when $\ell (w, \lambda)$  has a finite sum structure, as in the case of the regularized empirical risk, exploiting variance reduction techniques makes the convergence rate  $\rf(t)$ linear. In this situation the following assumption is made.

\begin{assumption}\label{ass:gradellestimator}
$\hat\ell$ is twice differentiable w.r.t. the first two arguments and such that, for every $w \in \R^d$, $\lambda \in \Lambda$, $j \in \{1,2\}$
\begin{equation*}\label{eq:exphessexp}
\begin{aligned}
    \EE[\nabla_1 \hat{\ell}(w,\lambda,\zeta)] = \nabla \ell(w,\lambda), \\
    \Exp{\grad_{1j}^2 \hat \ell(w, \lambda, \zeta)} =  \grad_{1j}^2  \ell(w, \lambda).
    \end{aligned}
\end{equation*}
Moreover, for every $w \in \R^d$, $\lambda \in \Lambda$ and $x \in \CZ$ there exists $L_\ell, m_\ell \geq 0$ such that:
\begin{equation*}\label{eq:tildelipvarloss}
    \norm{\grad_{11}^2 \hat \ell(w, \lambda, \zeta)} \leq  L_\ell \quad
    \Var{ \grad_{12}^2 \hat\ell(w,\lambda, \zeta)}
    \leq m_\ell.
\end{equation*}
\end{assumption}
If $\hat \ell$ satisfies
Assumption~\ref{ass:gradellestimator}, then $\hat \Phi (w,\lambda) = w - \alpha_\lambda \grad_1 \hat \ell(w,\lambda) $ satisfies
Assumption~\ref{ass:phiestimator}.
In addition, since $I -\jac_1\Phi(w,\lambda) =\alpha_\lambda \grad_1^2 \ell(w, \lambda)$ is a positive definite matrix, we have that the solution to the linear system \eqref{eq:linsys} can be written as
\begin{equation*}\label{eq:backminimization}
\begin{aligned}
    \vopt{w} &= \arg\min_v g(v; w, \lambda)\\
    g(v; w, \lambda) &:=  \frac{\alpha_\lambda}{2}v^\top  \grad_1^2 \ell(w, \lambda)v -v^\top \grad_1 E(w,\lambda),
    \end{aligned}
\end{equation*}
and $g(v;w,\lambda) = \EE[\hat g(v; w, \lambda, \zeta)]$, where
\begin{equation*}
    \hat g(v; w, \lambda, \zeta) = \frac{\alpha_\lambda }{2} v^\top \grad_1^2 \hat \ell(w, \lambda, \zeta)v -v^\top \grad_1 E(w,\lambda).
\end{equation*}
We can easily see that $g(\cdot; w, \lambda)$ is a strongly convex  quadratic function with Lipschitz smooth constant and modulus of strong convexity at least as good as the ones of $\alpha_\lambda \ell(\cdot, \lambda)$. Thus, we can solve both  subproblems in \Cref{algo1} using  the same stochastic optimization algorithm,
achieving the same theoretical performance for both rates $\rf(t)$ and $\hrf(k)$. %

We finally observe that the methods in \eqref{eq:inneralg}-\eqref{eq:backalg} can be rewritten as
\begin{align*}
    \wstoc[t+1] &: = \wstoc[t] - \sz{t} \alpha_\lambda  \grad_1 \hat \ell (\wstoc[t], \lambda, \zeta_t) %
    \\
    \vstoc[k+1]{w} &: = \vstoc[k]{w}  - \sz{k} \grad_1 \hat g(v; w, \lambda, \hat \zeta_k)
\end{align*}
which correspond to  SGD on $\alpha_\lambda \ell(w,\lambda)$ and on $g(v; w, \lambda)$ respectively.

\section{Experiments}\label{se:exp}

\newcommand{\figdim}{.245}
\newcommand{\figdimf}{.267}
\begin{figure*}[th]
\vspace{-.2truecm}
    \centering
    \hspace{-.2truecm}
    \includegraphics[width=\figdimf\textwidth]{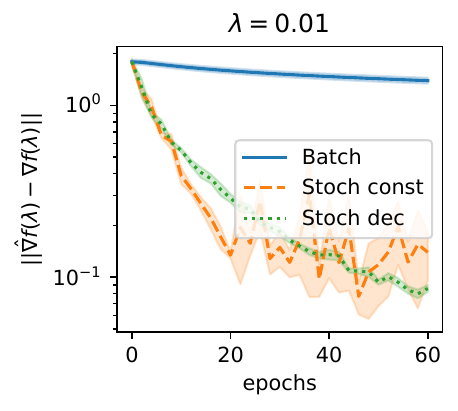}
    \hspace{-.25truecm}
    \includegraphics[width=\figdim\textwidth]{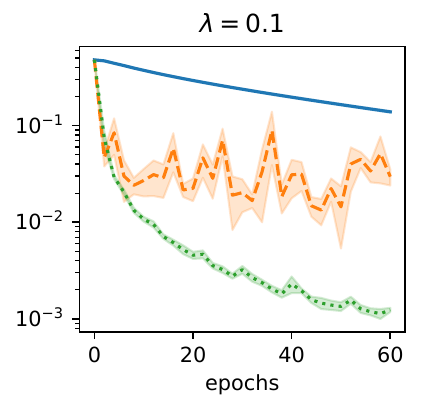}
    \hspace{-.25truecm}
    \includegraphics[width=\figdim\textwidth]{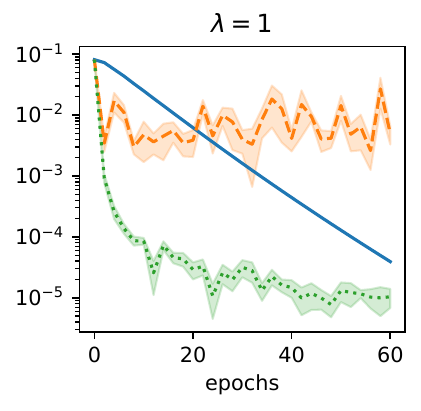}
    \hspace{-.25truecm}
    \includegraphics[width=\figdim\textwidth]{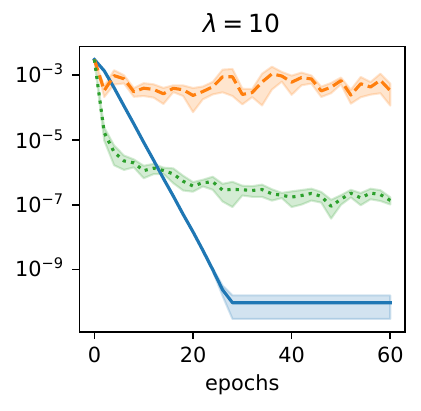}
    \caption{\small Experiment with a single regularization parameter. Convergence of three variants of SID for 4 choices of the regularization hyperparameter $\lambda \in \R_{++}$. %
    Here, 2 epochs refer, in the Batch version, to
    one iteration on the lower-level problem plus one iteration on the linear system, whereas, in the Stochastic versions, they refer to $100$ iterations on the lower-level problem plus $100$ iterations on the linear system.
    The plot shows mean (solid lines) and std (shaded regions) over 5 runs, which vary the train/validation splits and, for the stochastic methods, the order and composition of the  minibatches. \small}
    \label{fig:one}
\end{figure*}

\begin{figure}[h]
\vspace{-.2truecm}

    \centering
    \includegraphics[width=0.45\textwidth]{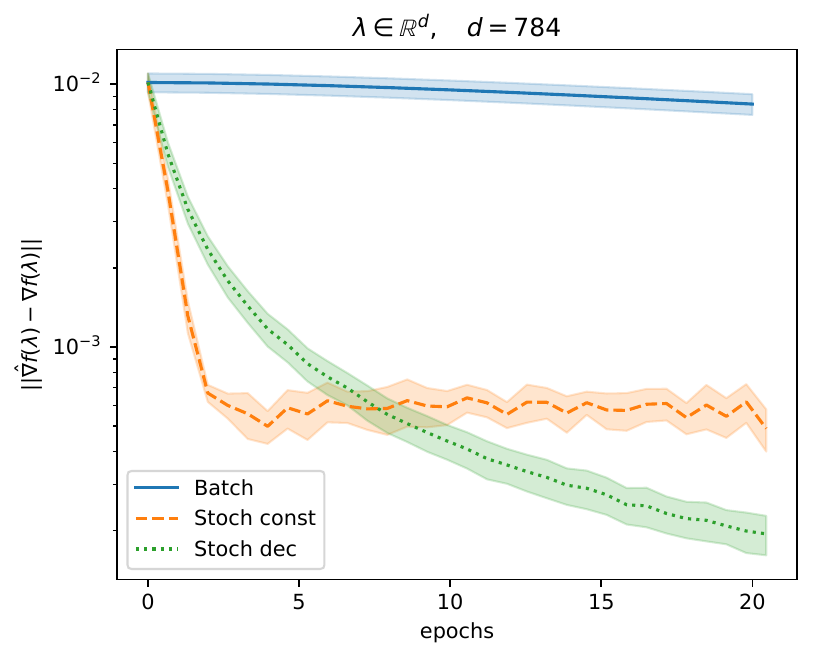}
    \caption{\small Experiment with multiple regularization parameters. Convergence of three variants of SID for several choices of the regularization hyperparameter $\lambda \in \R^d_{++}$. The plot shows mean (solid lines) and std (shaded regions) over 10 runs. For each run, $\lambda_i = e^{\epsilon_i}$, where $\epsilon_i \sim \mathcal{U}[-2, 2]$ for every $i \in \{1,\dots, d\}$. Epochs are defined as in \Cref{fig:one}.} 
    \label{fig:multi}
\end{figure}

In this section we present preliminary experiments evaluating the effectiveness of the SID method for estimating the hypergradient of $f$ in a real data scenario. In \Cref{se:addexp} we provide additional experiments on more realistc scenarios and with additional SID variants.
We focus on a hyperparameter optimization problem where we want to optimize the regularization parameter(s) in regularized logistic regression. Specifically, we consider a binary classification problem
with the aim to distinguish between odd and even numbers in the MNIST dataset.  Referring to problem~$\eqref{mainprobopt}$, we set

\newcommand{\nt}{n_{\text{tr}}}
\newcommand{\nv}{n_{\text{val}}}
\begin{align*}
    f(\lambda) &= \sum_{i=\nt+1}^{\nt+\nv} \psi(y_i x_i^\top w(\lambda)), \\
    w(\lambda) &= \arg\min_{w \in \R^d} \sum_{i=1}^{\nt} \psi(y_i x_i^\top w) + R(w, \lambda),
\end{align*}

where $\psi(u) = \log (1 + e^{-u})$ is the logistic loss, $(x_i, y_i)_{1 \leq i \leq \nt+\nv} \in (\R^p \times \{0,1\})^{\nt+\nv}$ are training and validation examples, and $R(w,\lambda)$ is set according to either of the two situations below
\begin{itemize}
    \item \textit{one regularization parameter}:
    
    $R(w,\lambda) =  \frac{\lambda}{2}\norm{w}^2$, $\lambda \in \R_{++}$
    \item \textit{multiple regularization parameters}: 
    
    $R(w, \lambda) = \frac{1}{2} w^\top {\rm diag}(\lambda) w$ where ${\rm diag}(\lambda)$ is the diagonal matrix formed by the elements of $\lambda \in \R^d_{++}$.
\end{itemize}
We set $\nt=\nv = 5000$, i.e., we pick 10000 examples from the MNIST training set and we group them into a training and a validation set of equal size. We set $\Phi$ %
to be the full gradient descent map on the lower-level objective, with optimal choice for the stepsize\footnote{We set the stepsize equal to two divided by the sum of the Lipschitz and strong convexity constants of the lower-level objective. This gives the optimal contraction rate $q_\lambda$.}. This map is a contraction because the lower objective is strongly convex and Lipschitz-smooth.
We test the following three variants of SID
(\Cref{algo1}), where we always solve the lower-level problem with $t$ iterations of the procedure~\eqref{eq:inneralg} and the linear system with $k=t$ iterations of the algorithm~\eqref{eq:backalg}. However, we make different choices for $\eta_{\lambda, t}$ and the estimator $\hat \Phi$.

{\bf Batch.} This variant of Algorithm~\ref{algo1} corresponds to the (deterministic) gradient descent algorithm with constant stepsize. We set $t_{\mathrm{Batch}} = k_{\mathrm{Batch}} = 30$ and, for every $t=0,\dots, t_{\mathrm{Batch}}$, 
$\eta_{\lambda,t} = \eta_{\lambda,k} = 1$ and $\hat \Phi(w, \lambda, \zeta) = \Phi(w,\lambda)$. 

{\bf Stoch const.} For this variant,  $\hat \Phi(w, \lambda, \zeta)$ corresponds to one step of stochastic gradient descent on a randomly sampled minibatch of 50 examples.  Thus, $t_{\mathrm{Stoch\ const}} = t_{\mathrm{Batch}} \times 100$, $k_{\mathrm{Stoch\ const}} = k_{\mathrm{Batch}} \times 100$, and
we pick $\eta_{\lambda, t} = \eta_{\lambda,k} = 1$,
for $t=0,\dots, t_{\mathrm{Stoch\ const}}$.

{\bf Stoch dec.} For this variant the estimator $\hat \Phi$ is the same as for the \emph{Stoch const} strategy, but we use decreasing stepsizes. More precisely, $\eta_{\lambda, t} = \eta_{\lambda,k} = \beta_\lambda/(\gamma_\lambda + t)$ with $\beta_\lambda = 2/(1-q_\lambda^2)$ and $\gamma_\lambda = \beta_\lambda$. Moreover, as before, $t_{\mathrm{Stoch\ dec}} = t_{\mathrm{Batch}} \times 100$, $k_{\mathrm{Stoch\ dec}} = k_{\mathrm{Batch}} \times 100$.

We note that the \emph{Batch} strategy is exactly the \textit{fixed point method} described by \cite{grazzi2020iteration}, which converges linearly to the true hypergradient. Moreover, for the stochastic versions, we can write $\hat \Phi(w,\lambda, \zeta) = \Phi_1(w,\zeta) + \Phi_2(w, \lambda)$, so that we are  in the case discussed in \Cref{rem:sig2} and hence, $\sigma_2 = \mtwo = 0$. In this situation, it follows from Remark~\ref{rmk:20201014a} that the \emph{Stoch dec} version of Algorithm~\ref{algo1} converges in expectation to the true hypergradient with a rate $O(1/(\gamma_\lambda + t))$,
whereas, according to Corollary~\ref{th:sgd} and Theorem~\ref{thm:finalbound}, the \emph{Stoch const} version can possibly approach the true hypergradient in a first phase (at linear rate), but ultimately might not converge to it.

In \Cref{fig:one,fig:multi} we show the squared error between the approximate and the true hypergradient ($\hat \nabla f(\lambda)$ and $\nabla f(\lambda)$ respectively) for the two regularization choices described above\footnote{Since for regularized logistic regression, the hypergradient is not available in closed form, we compute it by using the Batch version with $t=k=2000$ ($4000$ epochs in total).}.  In both figures we can see 
the effectiveness of the proposed SID method (and especially the \emph{Stoch dec} variant) against its deterministic version (AID) previously studied in \citep{grazzi2020iteration}.

\section{Conclusions and Future Work}

In this paper we studied a stochastic method for the approximation of the hypergradient in  bilevel problems defined through a fixed-point equation of a contraction mapping.
Specifically, we presented a stochastic version of the approximate implicit differentiation technique (AID), which is one of the most effective solutions for hypergradient computation as recently shown in \citep{grazzi2020iteration}. 
Our strategy (SID) estimates the hypergradient with the aid of two stochastic solvers in place of the deterministic solvers used in AID.  
We presented a formal description and a theoretical analysis of SID, ultimately providing a bound for the mean square error of the corresponding hypergradient estimator. As a byproduct of the analysis, we provided an extension of the SGD algorithm for stochastic fixed-point equations.
We have also conducted numerical experiments which confirm that using stochastic instead of deterministic solvers in SID can indeed yield a more accurate hypergradient approximation.

We believe that our analysis of stochastic fixed-point algorithms can be further extended to include variance reduction strategies and other advances commonly used for SGD. A good starting point for this extension can be the work by \cite{gorbunov2020unified}, which provides a unified theory for SGD methods in the strongly convex setting. 
Another promising direction would be the analysis of an overall bilevel optimization procedure using SID to approximate the hypergradient, which we have not addressed in the present work.

\section*{Acknowledgment}

We would like to thank Naoyuki Terashita for carefully reading the manuscript and pointing out a mistake (now fixed yielding the same result): we incorrectly used \Cref{lem:varprop}\ref{lem:varprop_i} on a matrix instead of a vector in the proof of \Cref{th:firstconv}.

\bibliographystyle{apalike}
\bibliography{ref}

\newpage
\appendix

\newtheorem{innercustomthm}{Theorem}
\newenvironment{customthm}[1]
  {\renewcommand\theinnercustomthm{#1}\innercustomthm}
  {\endinnercustomthm}

\onecolumn
\appendix

\vspace{3ex}
\begin{center}
{\Large\textbf{Supplementary Material}}
\end{center}
\vspace{2ex}

The supplementary material is organized as follows. \Cref{se:mainproofs} contains the proofs for the results presented in the paper. In \Cref{se:standardlemams} we  provide statements and proofs for some standard lemmas which are instrumental for the main results.
For convenience of the reader, before each proof we also restate the corresponding theorem. Finally, in \Cref{se:addexperiments} we present additional experiments.

\section{Main Proofs}\label{se:mainproofs}

\begin{lemma}\label{lm:normv}
Let \Cref{ass:aid} be satisfied. Then, for every $w \in \R^d$
\begin{equation}
\begin{aligned}
\norm{\vopt{w}} 
\leq \norm{(I - \jac_1\Phi(w, \lambda)^\top)^{-1}} \norm{\grad_1 \fo(w,\lambda)} 
\leq \frac{\Bo}{1 - \q}.
\end{aligned} 
\end{equation}
\end{lemma}
\begin{proof}
It follows from \eqref{eq:linsys} (the definition of $\vopt{w}$) and  Assumptions~\ref{ass:aid}\ref{ass:contraction} and \ref{ass:aid}\ref{ass:lipE}
\end{proof}

\begin{lemma}\label{lm:lipv}
Let \Cref{ass:aid} be satisfied. Then, for every $w \in \R^d$
\begin{equation}
    \norm{\vopt{w(\lambda)} -\vopt{w}} \leq  \left(\frac{\rhow\Bo}{(1-\q)^2} + \frac{\Low}{1-\q} \right)\norm{w(\lambda) - w}.
\end{equation}
\end{lemma}
\begin{proof}
Let $A_1 := (I- \jac_1 \Phi(w(\lambda), \lambda)^\top)$ and $A_2 = (I- \jac_1 \Phi(w, \lambda)^\top)$.
Then it follows from \Cref{lm:matrixinverse} that
\begin{align*}
    \norm{\vopt{w(\lambda)} -\vopt{w}}
    &\leq \norm{\grad_1 \fo(w(\lambda), \lambda)}\norm{A_1^{-1} - A_2^{-1}} + \Low \norm{A_2^{-1}}  \norm{w(\lambda) - w} \\
    &\leq \norm{\grad_1 \fo(w(\lambda), \lambda)}\norm{A_1^{-1}(A_2 -  A_1)A_2^{-1}} + \frac{ \Low}{1-\q} \norm{w(\lambda) - w} \\
    &\leq \left(\frac{\rhow}{(1-\q)^2}\norm{\grad_1 \fo(w(\lambda),\lambda)} + \frac{\Low}{1-\q} \right)\norm{w(\lambda) - w}.
\end{align*}
Moreover, Assumption~\ref{ass:aid} yields that $\norm{\grad_1 \fo(w(\lambda),\lambda)} \leq \Bo$.
Hence the statement follows.
\end{proof}

\subsection{Proofs of Section~\ref{se:MSE}}

\begin{customthm}{3.1}
Suppose that Assumptions~\ref{ass:aid},\ref{ass:phiestimator}, and \ref{ass:innerbackrates}
are satisfied. Let $\lambda \in \Lambda$, $t,k \in \N$ and set
\begin{equation*}
\hat \Delta_w := \norm{\wstoc - w(\lambda)},\ 
\LPhi := \norm{\jac_2 \Phi(w(\lambda), \lambda)},\  
    c_{1,\lambda} = \Lol + \frac{\Low\LPhi + \rhol\Bo}{1-\q} + \frac{\rhow\Bo\LPhi}{(1-\q)^2}.
\end{equation*}
Then the following hold.
\begin{enumerate}[label={\rm(\roman*)}]
\item
$\norm[\big]{ \Exp{\hat \grad f(\lambda) \given \wstoc}   - \grad f(\lambda)}\leq c_{1,\lambda}\hat \Delta_w
+ \LPhi \sqrt{\hrf(k)} + \rhol\hat \Delta_w \sqrt{\hrf(k)}$.
\item
$\norm{\EE[\hat \grad f(\lambda)] - \grad f(\lambda)}\leq c_{1,\lambda} \sqrt{\rf(t)} + \LPhi  \sqrt{\hrf(k)} + \rhol \sqrt{\rf(t)}\sqrt{\hrf(k)}.$
\end{enumerate}
\end{customthm}
\begin{proof}
\ref{th:boundbias_i}:
Using the definition of the approximate hypergradient and the fact that $\zeta$ nnd $\vstoc{\wstoc}$
are independent random variables, we get
\begin{equation*}
    \Exp{\hat \grad f(\lambda) \given \wstoc} = \grad_2 \fo(\wstoc, \lambda) + \jac_2 \Phi(\wstoc, \lambda)^\top \Exp{\vstoc{\wstoc}\given \wstoc}.
\end{equation*}
Consequently, recalling \eqref{eq:gradfv2}, we have,
\begin{align}
\nonumber\norm[\big]{ \Exp{\hat \grad f(\lambda) \given \wstoc}   - \grad f(\lambda)} 
\nonumber&\leq \norm{\grad_2 \fo(w(\lambda), \lambda) - \grad_2 \fo(\wstoc, \lambda)} \\[1ex]
\nonumber&\qquad+ \big\lVert\jac_2 \Phi(w(\lambda), \lambda)^\top \vopt{w(\lambda)} 
- \jac_2 \Phi(\wstoc, \lambda)^\top \Exp{\vstoc{\wstoc}\given \wstoc} \big\rVert\\[1ex]
\nonumber&\leq \norm{\grad_2 \fo(w(\lambda), \lambda) - \grad_2 \fo(\wstoc, \lambda)} \\[1ex]
\nonumber&\qquad+ \norm{\jac_2 \Phi(w(\lambda), \lambda)}\norm{\vopt{w(\lambda)}
     - \Exp{\vstoc{\wstoc}\given \wstoc}} \\[1ex]
\label{eq:20200601d}    &\qquad+ \norm{\jac_2 \Phi(w(\lambda), \lambda) - \jac_2 \Phi(\wstoc, \lambda)} \norm{\Exp{\vstoc{\wstoc}\given \wstoc}}.
\end{align}
Now, concerning the term $\norm{\vopt{w(\lambda)} - \Exp{\vstoc{\wstoc} \given \wstoc}}$ in the above inequality,
we have
\begin{align}
\nonumber\lVert \vopt{w(\lambda)} - &\Exp{\vstoc{\wstoc} \given \wstoc}\rVert \\[1ex]
\label{eq:20200601b}&\leq \norm{\vopt{w(\lambda)} - \vopt{\wstoc}} + \norm{\vopt{\wstoc} - \Exp{\vstoc{\wstoc}\given \wstoc}}.
\end{align}
Moreover, using  Jensen inequality and \Cref{ass:innerbackrates} we obtain
\begin{align}
\nonumber \norm{\vopt{\wstoc} - \Exp{\vstoc{\wstoc}\given \wstoc}} &=
\sqrt{ \norm{\Exp{\vopt{\wstoc}  - \vstoc{\wstoc}\given \wstoc}}^2 }\\[1ex]
\nonumber    & \leq \sqrt{\Exp{\norm{\vopt{\wstoc}  - \vstoc{\wstoc}}^2 \given \wstoc}}\\[1ex]
\label{eq:20200601c}    & \leq \sqrt{\hrf(k)}.
\end{align}
Therefore, using \Cref{lm:lipv},  \eqref{eq:20200601b} yields
\begin{equation}
\label{eq:20200601e}
\norm{\vopt{w(\lambda)} - \Exp{\vstoc{\wstoc} \given \wstoc}} \leq 
\left(\frac{\rhow\Bo}{(1-\q)^2} + \frac{\Low}{1-\q} \right)\norm{w(\lambda) - \wstoc} + \sqrt{\hrf(k)}.
\end{equation}
In addition, it follows from \eqref{eq:20200601c} and \cref{lm:normv} that
\begin{align}
\label{eq:20200601f}
\nonumber     \norm{\Exp{\vstoc{\wstoc}\given \wstoc}} &\leq \norm{\vopt{\wstoc}} 
     + \norm{\vopt{\wstoc}- \Exp{\vstoc{\wstoc}\given \wstoc}} \\
     &\leq \frac{\Bo}{1-\q} + \sqrt{\hrf(k)}.
\end{align}
Finally, combining \eqref{eq:20200601d}, \eqref{eq:20200601e}, and \eqref{eq:20200601f}, and
using \Cref{ass:aid}, \ref{th:boundbias_i} follows.
Then, since
\begin{equation*}
    \norm{\EE[\hat\grad f(\lambda)] - \grad f(\lambda)} = \norm[\big]{ \Exp[\big]{ \Exp{\hat \grad f(\lambda)\given \wstoc}   - \grad f(\lambda)}}
   \leq  \Exp[\big]{ \norm[\big]{ \Exp{\hat \grad f(\lambda)\given \wstoc}   - \grad f(\lambda)}},
\end{equation*}
 item \ref{th:boundbias_ii} follows by taking the expectation in \ref{th:boundbias_i}
and using \Cref{ass:innerbackrates} and that 
$\Exp{\hat \Delta_w} = \sqrt{(\Exp{\hat \Delta_w})^2} 
\leq \sqrt{\Exp{\hat \Delta^2_w}} \leq \sqrt{\rf(t)}$.

\end{proof}
\begin{customthm}{3.2}
Suppose that Assumptions~\ref{ass:aid},\ref{ass:phiestimator}, and \ref{ass:innerbackrates}
are satisfied. Let $\lambda \in \Lambda$, $t,k \in \N$ and set $\LPhi := \norm{\jac_2 \Phi(w(\lambda), \lambda)}$. 
Then 
\begin{equation}
\Exp[\big]{\Var{\hat \grad f(\lambda) \!\given\! \wstoc}} \leq 2\frac{\mtwo\Bo^2}{(1-\q)^2} 
+ 2(\LPhi^2 + \mtwo)\hrf(k) 
+ 2\rhol^2 \rf(t)\hrf(k).
\end{equation}
\end{customthm}
\begin{proof}
Let $\Exptilde{\cdot} := \Exp{\,\cdot \given \wstoc}$ and  $\Vartilde{\cdot} : = \Var{\,\cdot \given \wstoc}$. 
Then,
\begin{align*}
\Vartilde{\hat \grad f(\lambda)} &= \Exptilde[\big]{\norm{\hat \grad f(\lambda) - \Exptilde{\hat \grad f(\lambda)}}^2} \\
&= \Exptilde[\big]{\norm{\jac_2 \Phi(\wstoc, \lambda)^\top\Exptilde{\vstoc{\wstoc}} - \jac_2 \hat \Phi(\wstoc, \lambda, \zeta)^\top \vstoc{\wstoc}}^2} \\
&\leq \norm{\jac_2 \Phi(\wstoc, \lambda)}^2 \Exptilde[\big]{\norm{\vstoc{\wstoc} - \Exptilde{\vstoc{\wstoc}}}^2} \\
&\quad+ \Exptilde[\big]{\norm{\vstoc{\wstoc}}}^2\Exptilde[\big]{\norm{\jac_2 \hat \Phi(\wstoc, \lambda, \zeta)- \jac_2 \Phi(\wstoc, \lambda)}}^2.
\end{align*}
where for the last inequality we used that $\zeta \indep \vstoc{\wstoc}\given \wstoc$ and, in virtue of 
Lemma~\ref{lemA4},  that
\begin{equation*}
\Exptilde[\big]{(\vstoc{\wstoc} - \Exptilde{\vstoc{\wstoc}})^\top \jac_2 \Phi(\wstoc, \lambda)
    (\jac_2 \hat \Phi(\wstoc, \lambda, \zeta)- \jac_2 \Phi(\wstoc, \lambda))^\top \vstoc{\wstoc}} = 0.
\end{equation*}
In the following, we will bound each term of the inequality in order.
\begin{align*}
    \norm{\jac_2 \Phi(\wstoc, \lambda)}^2 &= \norm{\jac_2 \Phi(\wstoc, \lambda) \mp \jac_2 \Phi(w(\lambda), \lambda)}^2 \\
    &\leq 2\norm{\jac_2 \Phi(w(\lambda), \lambda)}^2 + 2\norm{\jac_2 \Phi(w(\lambda), \lambda) - \jac_2 \Phi(\wstoc, \lambda)}^2 \\
    &\leq 2\LPhi^2 + 2\rhol^2\norm{w(\lambda) - \wstoc}^2
\end{align*}
Then, applying \Cref{ass:innerbackrates}, and Lemma~\ref{lem:varprop}\ref{lem:varprop_ii}
\begin{align*}
    \Exptilde{\norm{\vstoc{\wstoc} - \Exptilde{\vstoc{\wstoc}}}^2} 
    &= \Vartilde{\vstoc{\wstoc}}\\
   &\leq \Exptilde{\norm{\vstoc{\wstoc}  - \vopt{\wstoc}}^2} \leq \hrf(k).
\end{align*}
Furthermore, exploiting \Cref{ass:aid} and \ref{ass:innerbackrates}, and \Cref{lm:normv},
\begin{align*}
    \Exptilde[\big]{\norm{\vstoc{\wstoc}}^2} &= \Exptilde[\big]{\norm{\vstoc{\wstoc} \mp \vopt{\wstoc}}}^2 \\
    &\leq 2\norm{\vopt{\wstoc}}^2 + 2\Exptilde[\big]{\norm{\vopt{\wstoc} - \vstoc{\wstoc}}^2} \\
    &\leq 2\frac{\Bo^2}{(1-\q)^2} + 2\hrf(k).
\end{align*}
The remaining term is bounded by $\mtwo$ through \Cref{ass:phiestimator}. Combining the previous bounds together and defining $\hat \Delta_w : = \norm{w(\lambda) - \wstoc}$ we get that
\begin{align*}
    \Vartilde{\hat \grad f(\lambda)} \leq 2\frac{\mtwo\Bo^2}{(1-\q)^2} + 2(\LPhi^2 + \mtwo)\hrf(k) + 2\rhol^2\hat \Delta^2_w\hrf(k) 
\end{align*}
The proof is completed by taking the total expectation on both sides of the inequality above.
\end{proof}

\begin{customthm}{3.3}
Suppose that Assumptions~\ref{ass:aid},\ref{ass:phiestimator}, and \ref{ass:innerbackrates}
are satisfied. Let $\lambda \in \Lambda$, and $t,k \in \N$. Then 
\begin{equation}
    \Var{\Exp{\hat \grad f(\lambda) \given \wstoc}} \leq 
    3 \big( c_{1,\lambda}^2\rf(t) +\LPhi^2\hrf(k)
    + \rhol^2 \rf(t)\hrf(k) \big),
\end{equation}
where $c_{1,\lambda}$ and $\LPhi$ are defined as in \Cref{th:boundbias}.
\end{customthm}
\begin{proof}
We derive from \Cref{lem:varprop}\ref{lem:varprop_ii} that
\begin{equation*}
    \Var{\Exp{\hat \grad f(\lambda) \given \wstoc}} 
\leq \Exp[\big]{\norm{\Exp{\hat \grad f(\lambda)\given \wstoc} - \grad f(\lambda)}^2}.
\end{equation*}
The statement follows from \Cref{th:boundbias}\ref{th:boundbias_i}, the inequality
$(a+b+c)^2 \leq 3 (a^2 + b^2 + c^2)$, and then by taking the total expectation and using \Cref{ass:innerbackrates}.
\end{proof}

\subsection{Proofs of Section~\ref{sec:fixed-point}}

\begin{customthm}{4.1}[Constant step-size]
Let Assumption~\ref{ass:inner} hold and 
suppose that $\eta_t = \eta \in \R_{++}$, for every $t \in \N$, and that
\begin{equation*}
    \eta \leq \frac{1}{1+\sigma_2}.
\end{equation*}
Let $(\x_t)_{t \in \N}$ be generated according to algorithm \eqref{eq:inneralgo} and set
 $\MSEW{t} : = \Exp{\norm{\x_t - \x^*}^2}$. Then, for all $t \in \N$,
\begin{equation}\tag{16}
\MSEW{t} \leq (1-\eta(1-q^2))^t\bigg(\MSEW{0} -\frac{\eta \sigma_1}{1-q^2} \bigg) + \frac{\eta \sigma_1}{1-q^2}.
\end{equation}
In particular, $\lim_{t \to \infty} \MSEW{t} \leq \eta \sigma_1/(1 - q^2)$.
\end{customthm}
\begin{proof}
\newcommand{\Expt}[2][]{\mathbb{E}_{t}#1[{#2}#1]}
Let $\sW_t$ be the $\sigma$-algebra generated by $\x_0,\x_1, \cdots, \x_t$. Then
\begin{align*}
\EE[\norm{\x_{t+1} - \x^*}^2 \,\vert\, \sW_t]
 &= \EE[ \norm{\x_t - \x^* + \eta (\hat \T(\x_t,\zeta_t) - \x_t)}^2 \,\vert\,\sW_k] \\
&= \norm{\x_t - \x^*}^2 + \eta^2\EE[ \norm{ (\hat \T(\x_t, \zeta_t) \mp \T(\x_t) - \x_t)}^2 \,\vert\, \sW_k] \\
&\quad + 2\eta(\x_t - \x^*)^\top(\T(\x_t)-\x_t) \\
&= \norm{\x_t - \x^*}^2  + \eta^2\norm{\T(\x_t) - \x_t}^2
+ \eta^2 \VV[\hat \T(\x_t,\zeta_t) \,\vert\, \sW_t] \\
&\quad+ 2\eta(\x_t - \x^*)^\top(\T(\x_t)-\x_t) \\
&\leq \norm{\x_t - \x^*}^2  + \eta^2(1 + \sigma_2)\norm{\T(\x_t) \mp \x^* - \x_t}^2
+ \eta^2\sigma_1  \\
&\quad+ 2\eta(\x_t - \x^*)^\top(\T(\x_t) \mp \x^* -\x_t) \\
&\leq (1-2\eta + \eta^2(1 + \sigma_2))\norm{\x_t - \x^*}^2  + \eta^2(1 + \sigma_2)\norm{\T(\x_t) - \x^*}^2
+ \eta^2\sigma_1  \\
&\quad+ \eta(1-\eta(1 + \sigma_2))2(\x_t - \x^*)^\top(\T(\x_t) - \x^*). 
\end{align*}
Furthermore, since $\norm{\T(\x_t) - \x^*} \leq q\norm{\x_t - \x^*}$ and $2ab \leq a^2+b^2$, we have that
\begin{align*}
    2(\x_t - \x^*)^\top(\T(\x_t) - \x^*) \leq 2\norm{\x_t - \x^*}\norm{\T(\x_t) - \x^*} \leq (1+ q^2)\norm{\x_t -\x^*}^2.
\end{align*}
From the upper bound on the step size we have that $\eta(1-\eta(1 + \sigma_2)) \geq 0$, hence:
\begin{align}
\nonumber    \EE [ \norm{\x_{t+1} - \x^*}^2 \,\vert\,\sW_t]&\leq (1-2\eta )\norm{\x_t - \x^*}^2 + \eta^2(1 + \sigma_2)(1+q^2)\norm{\x_t - \x^*}^2
+ \eta^2\sigma_1  \\
\nonumber&\quad+\eta(1+q^2)\norm{\x_t - \x^*}^2 - \eta^2(1 + \sigma_2)(1+q^2)\norm{\x_t - \x^*}^2 \\
&\leq (1 - \eta(1-q^2))\norm{\x_t - \x^*}^2 + \eta^2\sigma_1.
\end{align}
Taking total expectations we get 
\begin{equation*}
   \EE[\norm{\x_{t+1} - \x^*}^2] \leq (1 - \eta(1-q^2))\Exp{\norm{\x_t - \x^*}^2} + \eta^2\sigma_1
\end{equation*}
and subtracting $\eta \sigma_1/(1-q^2)$ from both sides we obtain
\begin{equation}\label{eq:stepcontraction}
   \EE[\norm{\x_{t+1} - \x^*}^2] - \frac{\eta \sigma_1}{1-q^2}  \leq (1 - \eta(1-q^2))\bigg(\Exp{\norm{\x_t - \x^*}^2} - \frac{\eta \sigma_1}{1-q^2}\bigg).
\end{equation}
Now the statement follows by applying the above inequality recursively.
\end{proof}

\begin{customthm}{4.2}[Decreasing step-sizes]
Let Assumption~\ref{ass:inner} hold and 
suppose that for every $t \in \N$
\begin{equation}
\tag{20}
    \eta_t \leq \frac{1}{1+\sigma_2}, \quad
    \sum_{t=1}^\infty \eta_t = \infty, \quad 
    \sum_{t=1}^\infty \eta^2_t < \infty.
\end{equation}
Let $(\x_t)_{t \in \N}$ be generated according to Algorithm~\eqref{eq:inneralgo}.
Then 
\begin{equation*}
    \x_t \to \x^* \qquad \PP\text{-a.s.}
\end{equation*}
Moreover, if $\eta_t = \beta/(\gamma  + t)$, 
with $\beta > 1/(1-q^2)$ and  $\gamma \geq \beta(1 + \sigma_2)$,
then we have
\begin{equation}\tag{21}
    \Exp{\norm{\x_t - \x^*}^2} 
    \leq \frac{c}{\gamma + t},
\end{equation}
where 
\begin{equation*}
    c : = \max \left\{\gamma \Exp{\norm{\x_0 - \x^*}^2}, \frac{\beta^2\sigma_1}{\beta(1-q^2) -1} \right\}.
\end{equation*}
\end{customthm}
\begin{proof}
\newcommand{\Expt}[2][]{\mathbb{E}_{t}#1[{#2}#1]}
As in the proof of Theorem~\ref{th:innerconvergence2} we get
\begin{equation}
\label{eq:20200528a}
(\forall\, t \in \N)\quad
\EE [ \norm{\x_{t+1} - \x^*}^2 \,\vert\,\sW_t]
\leq (1 - \eta_t(1-q^2))\norm{\x_t - \x^*}^2 + \eta^2_t\sigma_1.
\end{equation}
Taking total expectations we obtain
\begin{equation}\label{eq:quasidescent}
(\forall\, t \in \N)\quad
   \EE[\norm{\x_{t+1} - \x^*}^2] \leq (1 - \eta_t(1-q^2))\Exp{\norm{\x_t - \x^*}^2} + \eta^2_t\sigma_1,
\end{equation}
which can be equivalently written as
\begin{equation*}
(\forall\, t \in \N)\quad
    (1-q^2) \eta_t \EE[\norm{\x_{t} - \x^*}^2] \leq 
    \EE[\norm{\x_{t} - \x^*}^2] - \EE[\norm{\x_{t+1} - \x^*}^2]
    + \eta_t^2 \sigma_1.
\end{equation*}
Since the right hand side is summable (being the sum of a telescopic series and a summable series), we have
\begin{equation}
\label{eq:20201013a}
    (1-q^2)\sum_{t=0}^\infty \eta_t \Exp{\norm{\x_t - \x^*}^2} \leq \EE[\norm{\x_{0} - \x^*}^2] + \sigma_1 \sum_{t=0}^{+\infty} \eta_t^2<+\infty.
\end{equation}
Now, it follows from \eqref{eq:20200528a} that $(\norm{\x_t - \x^*}^2)_{t \in \N}$ is an almost supermartingale (in the sense of \cite{ROB1971}), hence
$\norm{\x_t - \x^*}^2 \to \zeta$ $\PP$-a.s.~for some positive random variable $\zeta$.
Since $\sum_{t=0}^{+\infty} \eta_t = +\infty$, it follows from \eqref{eq:20201013a} that
$\liminf_{t\to+\infty} \Exp{\norm{\x_t - \x^*}^2}=0$. Then Fatou's lemma yields that
$\EE[\zeta] \leq \liminf_{t\to +\infty}\EE[\norm{\x_t - \x^*}^2] =0$.
Thus, since $\zeta$ is positive, $\zeta=0$ $\PP$-a.s.
and hence $x_t \to x_*$ $\PP$-a.s.

Concerning the second part of the statement, it is easy to see that the sequence $(\eta_t)_{t \in \N}$ satisfies the assumptions \eqref{eq:20201013b}. We can thus apply \cref{eq:quasidescent} at each iteration. Let $\Delta_t := \Exp{\norm{\x_t - \x^*}^2}$, from the definition of $c$ we have that for $t=0$ $\Delta_0 \leq c/\gamma$.
Now, suppose that \eqref{eq:ratediminish} holds at step $t$. We want to prove that it holds at $t+1$. Defining $\xi := \gamma + t$, it follows from \eqref{eq:quasidescent} that
\begin{align*}
    \Delta_{t+1} &\leq  \frac{\xi -\beta(1 -q^2)}{\xi} \frac{c}{\xi} + \frac{\beta^2\sigma_1}{\xi^2}  \\
    &= \frac{(\xi -1)}{\xi^2}c - \underbrace{\frac{(\beta(1-q^2)-1)c -\beta^2\sigma_1}{\xi^2}}_{\geq 0 \text{ by the definition of }c \text{ and } \beta}  \\
    &\leq \frac{c}{\xi + 1} = \frac{c}{\gamma + t + 1},
\end{align*}
where the last inequality derives from $\xi^2 \geq (\xi-1)(\xi+1)$.
\end{proof}

\liptovar*
\begin{proof}
Let  $\x \in \R^d$. Then by \Cref{ass:inner}-\ref{ass:tildeexp} and the inequality $\norm{a+b}^2 \leq 2\norm{a}^2 + 2\norm{b}^2$ we get
\begin{align*}
    \Var{\hat \T(\x, \zeta)} &= \Exp{\norm{\hat \T(\x, \zeta) \mp \hat \T(\x^*, \zeta) - \T(\x)}^2} \\
    &\leq 2 \Exp{\norm{\hat \T(\x, \zeta) - \hat \T(\x^*, \zeta)}^2} + 2 \Exp{\norm{\hat \T(\x^*, \zeta) - \T(\x)}^2 } \\
    &\leq 2 \Exp{\norm{\hat \T(\x, \zeta) - \hat \T(\x^*, \zeta)}^2} + 2\Var{\hat \T(\x^*,\zeta)} + 2\norm{\T(\x^*) - \T(\x)}^2.
\end{align*}
Therefore, leveraging  \Cref{ass:inner}-\ref{ass:cont} and \Cref{ass:tildephilip} we have
\begin{align*}
     \Var{\hat \T(\x, \zeta)} &\leq 2(\LtT^2+ q^2)\norm{\x- \x^*}^2 + 2\Var{\hat \T(\x^*,\zeta)}. 
\end{align*}
Finally, note that $\norm{\x-\x^*} \leq \norm{\x-\T(\x)} + \norm{\T(\x)-\x^*} = \norm{\x-\T(\x)} + \norm{\T(\x)-\T(\x^*)} \leq \norm{\x-\T(\x)} + q \norm{\x-\x^*}$ and hence $\norm{\x-\x^*} \leq \norm{\x-\T(\x)}/(1-q)$. The statement follows.
\end{proof}

\subsection{Proofs of Section~\ref{sec:lower-level}}

\firstconv*
\begin{proof}
The statement follows by applying Theorem~\ref{th:innerconvergence} with 
 $\hat{T} =\hat \Phi(\cdot, \lambda, \cdot)$ and $\hat{T} =\hat \Psi_w(\cdot, \lambda, \cdot)$.
To that purpose, in view of Theorem~\ref{th:liptovar} it is sufficient to verify 
Assumptions~\ref{ass:inner}\ref{ass:cont}-\ref{ass:tildeexp} and \ref{ass:tildephilip}.
This is immediate for $\hat \Phi(\cdot, \lambda, \cdot)$,
due to Assumptions~\ref{ass:aid}\ref{ass:contraction} and \ref{ass:phiestimator}.
Concerning $\hat \Psi_w(\cdot, \lambda, \cdot)$, it follows from Assumptions~\ref{ass:aid}\ref{ass:contraction} and \ref{ass:phiestimator}, that, for every $w,v, v_1, v_2 \in \R^d$ and for every $x \in \CZ$,
\begin{align*}
    \Exp{\hat \Psi_w(v, \lambda, \zeta)} &=  \jac_1 \Phi(w, \lambda) v + \grad_1 E(w,\lambda) =: \Psi_w(v,\lambda) \\
    \norm{\Psi_w(v_1,\lambda) - \Psi_w(v_2,\lambda)} &\leq \norm{\jac_1 \Phi(w, \lambda)}\norm{v_1 - v_2} \leq \q \norm{v_1 - v_2}  \\
    \norm{\hat \Psi_w(v_1,\lambda, x) - \hat \Psi_w(v_2,\lambda, x)} &\leq \norm{\jac_1 \hat \Phi(w, \lambda, x)}\norm{v_1 - v_2} \leq \LtPhi \norm{v_1 - v_2}.
\end{align*}
Now, it remains just to compute the corresponding $\sigma_1$ in Theorem~\ref{th:liptovar},
which reduces to bound $\Var{\hat \Psi_w(\vopt{w}, \lambda, \zeta^\prime)}$.
To that purpose, recalling \eqref{eq:gradfv2}, we note that
\begin{align*}
    \Var{\hat \Psi_w(\vopt{w}, \lambda, \zeta^\prime)} &= \Exp[\big]{\norm{\hat \Psi_w(\vopt{w}, \lambda, \zeta^\prime)}^2} - \norm{\jac_1 \Phi(w, \lambda) v(w,\lambda)}^2  \\
    &\leq
    \norm{\vopt{w}}^2\Exp[\big]{\norm{\jac_1 \hat \Phi(w,\lambda, \zeta^\prime)}^2} - \norm{\jac_1 \Phi(w, \lambda) v(w,\lambda)}^2 \\
    &\leq \norm{\vopt{w}}^2 \LtPhi^2
\end{align*}
where we used \Cref{lem:varprop}\ref{lem:varprop_i} in the first equation, the sub-multiplicativity of the operator and euclidean norms and the fact that $v(w,\lambda)$ is not random in the second inequality, and finally, Assumption~\ref{ass:phiestimator}\ref{eq:expjacexp_iii} for the last inequality.

Therefore, using \Cref{lm:normv}, we have that $\Var{\hat \Psi_w(\vopt{w}, \lambda, \zeta^\prime)} \leq \LtPhi^2 \norm{\nabla_1 \fo (w,\lambda)}^2/(1-\q)^2$.
\end{proof}

\section{Standard Lemmas}\label{se:standardlemams}
\begin{lemma}
\label{lem:20200601a}
Let $X$ be a random vector with values in $\R^d$
and suppose that $\EE[\norm{X}^2]<+\infty$.
Then $\EE[X]$ exists in $\R^d$ and 
$\norm{\EE[X]}^2 \leq \EE[\norm{X}^2]$.
\end{lemma}
\begin{proof}
It follows from H\"older's inequality that
 $\EE[\norm{X}] \leq \EE[\norm{X}^2]$. Therefore
$X$ is Bochner integrable with respect to $\PP$ and
$\norm{\EE[X]} \leq \EE[\norm{X}]$.
Hence using Jensen's inequality we have
$\norm{\EE[X]}^2 \leq (\EE[\norm{X}])^2 \leq 
\EE[\norm{X}^2]$ and the statement follows. 
\end{proof}

\begin{definition}
Let $X$ be a random vector with value in $\R^d$ 
such that $\EE[\norm{X}^2]<+\infty$. Then
the variance of $X$ is
\begin{equation}\label{eq:variance}
    \Var{X} := \Exp{\norm{X - \Exp{X}}^2}
\end{equation}
\end{definition}

\begin{lemma}[Properties of the variance]
\label{lem:varprop}
Let $X$ and $Y$ be two independent random variables with values in $\R^d$
and let $A$ be a random matrix with values in $\R^{n\times d}$ which is independent on $X$. We also assume that $X,Y$, and $A$ have finite second moment. Then the following hold.
\begin{enumerate}[label={\rm (\roman*)}]
\item\label{lem:varprop_i} 
$\Var{X} = \EE[\norm{X}^2] - \norm{\EE[X]}^2$,
\item\label{lem:varprop_ii} 
For every $x \in \R^d$, $\EE[\norm{X- x}^2] = \Var{X} + \norm{\EE[X] - x}^2$. Hence,
$\Var{X} = \min_{x \in \R^d} \EE[\norm{X- x}^2]$,
\item\label{lem:varprop_iii} 
$ \Var{X + Y} = \Var{X} + \Var{Y}$,
\item\label{lem:varprop_iv} 
$\Var{AX}  \leq \Var{A}\Var{X} + \norm{\Exp{A}}^2\Var{X} + \norm{\Exp{X}}^2\Var{A}$.
\end{enumerate}
\end{lemma}
\begin{proof}
\ref{lem:varprop_i}-\ref{lem:varprop_ii}:
Let $x \in \R^d$. Then, $\norm{X - x}^2 = \norm{X- \EE[X]}^2 + \norm{\EE[X] - x}^2
+ 2(X- \EE[X])^\top(\EE[X]- x)$. Hence, taking the expectation
we get $\EE[\norm{X- x}^2] = \Var{X} + \norm{\EE[X] - x}^2$.
Therefore, $\EE[\norm{X- x}^2] \geq \Var{X}$ and for $x=\EE[X]$ we get $\EE[\norm{X- x}^2] = \Var{X}$. Finally, for $x=0$ we get \ref{lem:varprop_i}.

\ref{lem:varprop_iii}:
Let $\bar X := \Exp{X}$ and $\bar Y := \Exp{Y}$, we have
\begin{align*}
    \Var{X + Y} &= \Exp{\norm{X - \bar X + Y - \bar Y}^2} \\
    &= \Exp{\norm{X- \bar X}^2} + \Exp{\norm{Y- \bar Y}^2} + 2 \Exp{X - \bar X}^\top\Exp{Y - \bar Y} \\
    &= \Exp{\norm{X- \bar X}^2} + \Exp{\norm{Y- \bar Y}^2}
\end{align*}
Recalling the definition of $\Var{X}$ the statement follows.

\ref{lem:varprop_iv}:
Let $\bar A := \Exp{A}$ and $\bar X := \Exp{X}$.
Then,
\begin{align*}
\Var{AX} &= \Exp{\norm{AX - \Exp{A}\Exp{X}}^2} \\
&= \Exp{\norm{AX -A \bar X + A \bar X - \bar A \bar X}^2} \\
&= \Exp{\norm{A(X - \bar X) + (A  - \bar A) \bar X}^2} \\
&= \Exp{\norm{A(X - \bar X)}^2} + \Exp{\norm{(A  - \bar A) \bar X}^2} \\
&\quad + 2\Exp{(X- \bar X)^\top A^\top(A - \bar A)\bar X}\\
&= \Exp{\norm{A(X - \bar X)}^2} + \Exp{\norm{(A  - \bar A) \bar X}^2} \\
&\quad +  2\Exp{(X- \bar X)^\top} \Exp{A^\top(A - \bar A)\bar X}\\
&=\Exp{\norm{(A -\bar A + \bar A)(X - \bar X)}^2} + \Exp{\norm{(A  - \bar A) \bar X}^2} \\
&=\Exp{\norm{(A -\bar A )(X - \bar X)}^2} + \Exp{\norm{(A  - \bar A) \bar X}^2} + \Exp{\norm{\bar A(X  - \bar X)}^2}  \\
&\quad+ 2\Exp{(X - \bar X)^\top (A - \bar A)^\top \bar A (X - \bar X)} \\
&=\Exp{\norm{(A -\bar A )(X - \bar X)}^2} + \Exp{\norm{(A  - \bar A) \bar X}^2} + \Exp{\norm{\bar A(X  - \bar X)}^2}  \\
&\quad+ 2\Exp{(X - \bar X)^\top \Exp{A - \bar A \given X}^\top \bar A (X - \bar X)} \\
&\leq \Exp{\norm{A -\bar A}^2}\Exp{\norm{X - \bar X}^2} \\
&\quad + \Exp{\norm{A  - \bar A}^2} \norm{\bar X}^2 + \norm{\bar A}^2 \Exp{\norm{X  - \bar X)}^2}
\end{align*}
In the above equalities we have used the independence of $A$ and $X$ in the formulas $\EE[AX] = \EE[A]\EE[X]$,
$\EE[(X - \bar{X})^\top A^\top(A - \bar{A} \bar{X})] 
= \EE[(X - \bar{X})^\top] \EE[A^\top(A - \bar{A} \bar{X})]$, and $\EE[(X - \bar{X})^\top (A - \bar{A})^\top \bar{A}(X - \bar{X}) \,\vert\, X] = 
(X - \bar{X})^\top \EE[(A - \bar{A})^\top  \,\vert\, X]\bar{A}(X - \bar{X})$.
\end{proof}

\begin{lemma}
Let $f:\CZ \subset \R^n \mapsto \R^m$ be an $L$-Lipschitz function, with $L>0$, meaning that
\begin{equation*}
    \norm{f(x)-f(y)} \leq L \norm{x-y} \qquad \forall x,y \in \CZ
\end{equation*}
Let $X$ be a random variable with finite variance. Then, we have that
\begin{equation}
    \Var{f(X)} \leq L^2 \Var{X} 
\end{equation}
\end{lemma}
\begin{proof}
We have
\begin{align*}
\Var{f(X)} &= \Exp{\norm{f(X) - \Exp{f(X)}}^2} \\
&= \Exp{\norm{f(X) - f(\Exp{X)})}^2} - \norm{f(\Exp{X}) - \Exp{f(X)}}^2 \\
&\leq \Exp{\norm{f(X) - f(\Exp{X)})}^2} \\
&\leq L^2 \Exp{\norm{X - \Exp{X}}^2} = L^2\Var{X}.
\end{align*}
\end{proof}

\begin{definition}\textit{(Conditional Variance)}.
Let $X$ be a random variable with values in $\R^d$ 
and $Y$ be a random variable with values in a measurable space $\mathcal{Y}$. We call \emph{conditional variance} of $X$ given $Y$ the quantity
\begin{equation*}
    \Var{X \given Y} := \Exp{\norm{X -\Exp{X \given Y}}^2 \given Y}.
\end{equation*}{}
\end{definition}{}

\begin{lemma}{(Law of total variance)}\label{lm:totvariance}
Let $X$ and $Y$ be two random variables, we can prove that 
\begin{equation}
    \Var{X} = \Exp{\Var{X \given Y}} + \Var{\Exp{X \given Y}}
\end{equation}
\end{lemma}
\begin{proof}
\begin{align*}
    \Var{X} &= \Exp{\norm{X - \Exp{X}}^2} \\
    \text{(var. prop.)} \implies \quad 
    &= \Exp{\norm{X}^2} - \norm{\Exp{X}}^2 \\
    \text{(tot. expect.)} \implies \quad
    &= \Exp{\Exp{\norm{X}^2 \given Y}} - \norm{\Exp{\Exp{X \given Y}}}^2 \\
     \text{(var. prop.)} \implies \quad
     &= \Exp{\Var{X \given Y} + \norm{\Exp{X \given Y}}^2} - \norm{\Exp{\Exp{X \given Y}}}^2 \\
    &= \Exp{\Var{X \given Y}} + \left(\Exp{\norm{\Exp{X \given Y}}^2} - \norm{\Exp{\Exp{X \given Y}}}^2\right)
\end{align*}
recognizing that the term inside the parenthesis is the conditional variance of $\Exp{X \given Y}$ gives the result.
\end{proof}

\begin{lemma}
\label{lemA4}
Let $\zeta$ and $\eta$ be two independent random variables with values in $\CZ$ and $\mathcal{Y}$ respectively.
Let $\psi\colon \mathcal{Y} \to \R^{m\times n}, \phi\colon \CZ \to \R^{n\times p}$, and
$\varphi\colon \mathcal{Y} \to \R^{p\times q}$ matrix-valued measurable functions. Then
\begin{equation}
\EE[\psi(\eta) (\phi(\zeta) - \EE[\phi(\zeta)] )\varphi(\eta)] = 0
\end{equation}
\end{lemma}
\begin{proof}
Since, for every $y \in \mathcal{Y}$, $B \mapsto \psi(y) B \varphi(y)$ is linear and $\zeta$ and $\eta$ 
are independent, we have
\begin{equation*}
\EE[\psi(\eta) (\psi(\zeta) - \EE[\psi(\zeta)]) \varphi(\eta) \,\vert \eta] = \psi(\eta) \EE\big[\phi(\zeta) - \EE[\phi(\zeta)]\big] \varphi(\eta) =0.
\end{equation*}
Taking the expectation the statement follows.
\end{proof}

\begin{lemma}
\label{lm:matrixinverse}
Let $A$ be a square matrix such that $\norm{A}\leq q < 1$
Then, $I - A$ is invertible and
\begin{align*}
\norm{(I - A)^{-1}} \leq \frac{1}{1 - q}.
\end{align*}
\end{lemma}
\begin{proof}
Since $\norm{A} \leq q < 1$,
\begin{equation*}
\sum_{k=0}^\infty \norm{A}^k \leq \sum_{k=0}^\infty q^k 
= \frac{1}{1-q}.
\end{equation*}
Thus, the series $\sum_{k=0}^\infty A^k$ is convergent, say to $B$, and
\begin{equation}
    (I - A) \sum_{i=0}^k A^i = \sum_{i=0}^k A^i(I - A) = \sum_{i=0}^k A^i - \sum_{i=0}^{k+1} A^i + I \to I,
\end{equation}
so that $(I-A)B = B(I-A)=I$. Therefore, $I-A$ is invertible
with inverse $B$ and hence $\norm{(I-A)^{-1}} \leq \sum_{k=0}^\infty\norm{A}^k \leq 1/(1-q)$.
\end{proof}

\section{Additional Experiments}\label{se:addexperiments}
In this section we provide additional experiments in two of the settings outlined in \cite{grazzi2020iteration}. In addition to the three methods considered in \Cref{se:exp}, we also test variants of the algorithm which use a mixed Stochastic/Batch strategy for the solution of the two subproblems as well as variants for which $t \neq k$. To have a fair comparison, each method computes the approximate hypergradient using the same number of epochs.
We report the differences among the methods in \Cref{tb:methods}.

\begin{table}[h]
\centering
\caption{Differences among the methods used in the experiments. 
The column \textbf{\% epochs}, provides percentages of epochs used to solve the lower level problem (LL) and the linear system (LS), while the column \textbf{algorithm} indicates which method is used for each of the two subproblems: gradient descent (GD), stochastic gradient descent with constant step size (SGD const) and stochastic gradient descent with decreasing step sizes (SGD dec).}\label{tb:methods}
\hspace{0.01truecm}
\begin{tabular}{@{}lcc@{}}
\toprule
\textbf{Method}                & \textbf{\% epochs (LL, LS)} & \textbf{algorithm (LL, LS)} \\ \midrule
\textit{Batch}                 & 50, 50                      & GD, GD                      \\
\textit{Stoch const}           & 50, 50                      & SGD const, SGD const        \\
\textit{Stoch dec}             & 50, 50                      & SGD const, SGD const        \\
\textit{Stoch/Batch}           & 50, 50                      & SGD dec, GD                 \\
\textit{Batch/Stoch}           & 50, 50                      & GD, SGD dec                 \\
\textit{Batch 75\%/25\%}       & 75, 25                      & GD, GD                      \\
\textit{Stoch const 75\%/25\%} & 75, 25                      & SGD const, SGD const        \\
\textit{Stoch dec 75\%/25\%}   & 75, 25                      & SGD dec, SGD dec            \\ \bottomrule
\end{tabular}
\end{table}

For each method we set the number of iterations for the lower-level problem and linear system ($t$ and $k$) in Algorithm~\ref{algo1} as follows.
\begin{align*}
    t = \text{round}\left( \frac{\text{\% epochs LL}}{100} \times \text{total \# of  epochs} \times n_{tr} \div \text{batch size LL} \right) \\ 
    k = \text{round}\left(\frac{\text{\% epochs LS}}{100} \times \text{total \# of epochs} \times n_{tr} \div \text{batch size LS}  \right)
\end{align*}
where \% epochs LL/LS is the corresponding value in \Cref{tb:methods},  $n_{tr}$ is the number of examples in the training set and batch size LL (batch size LS) is the batch size used to solve the lower-level problem (linear system). The total number of epochs and $n_{tr}$ depend on the setting and are the same for all methods.

\subsection{Multinomial Regularized Logistic Regression on MNIST}\label{se:addexp}

We consider the following multinomial logistic regression setting on the MNIST dataset.
\newcommand{\CE}{\mathrm{CE}}
\begin{align*}
    f(\lambda) &= \sum_{i=\nt+1}^{\nt+\nv} \CE(y_i, W(\lambda) x_i), \\
    W(\lambda) &= \arg\min_{W \in \R^{c \times d}} \sum_{i=1}^{\nt} \CE(y_i, W x_i) + R(w, \lambda),
\end{align*}

where $c$ is the number of classes (10 for MNIST),  $\CE$ is the cross entropy loss, $(x_i, y_i)_{1 \leq i \leq \nt+\nv} \in (\R^d \times \{0,\dots, c\})^{\nt+\nv}$ are training and validation examples, and $R(w,\lambda)$ is set according to either of the two situations below
\begin{itemize}
    \item \textit{one regularization parameter}:
    
    $R(w,\lambda) =  \frac{\lambda}{2}\norm{w}^2$, $\lambda \in \R_{++}$
    \item \textit{multiple regularization parameters (one per feature)}: 
    
    $R(w, \lambda) = \frac{1}{2} \sum_{i=1}^c \sum_{j=1}^d \lambda_j w_{ij}^2 $ where $\lambda \in \R^d_{++}$.
\end{itemize}
In this scenario we take into account the whole MNIST training set containing 60 thousands examples, which we split in half to make the train and validation sets, i.e. $\nt=\nv=30000$. The batch size for the stochastic variants is 300.
 \Cref{fig:onemulti} shows the results. Even in this setting, the pure stochastic variants have a clear advantage over the Batch algorithm. We also note that the mixed strategies perform worse than the pure stochastic strategies and that
there is no particular gain in allocating more epochs to solve the lower-level problem.

\begin{figure*}[ht]
\vspace{-.2truecm}
    \centering
    \hspace{-.2truecm}
    \includegraphics[width=0.46\textwidth]{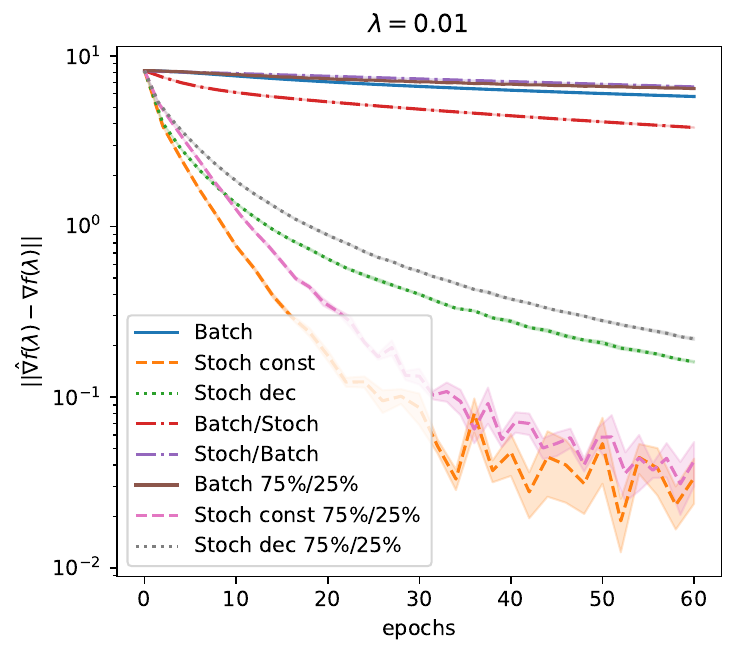}
    \hspace{-.25truecm}
    \includegraphics[width=0.44\textwidth]{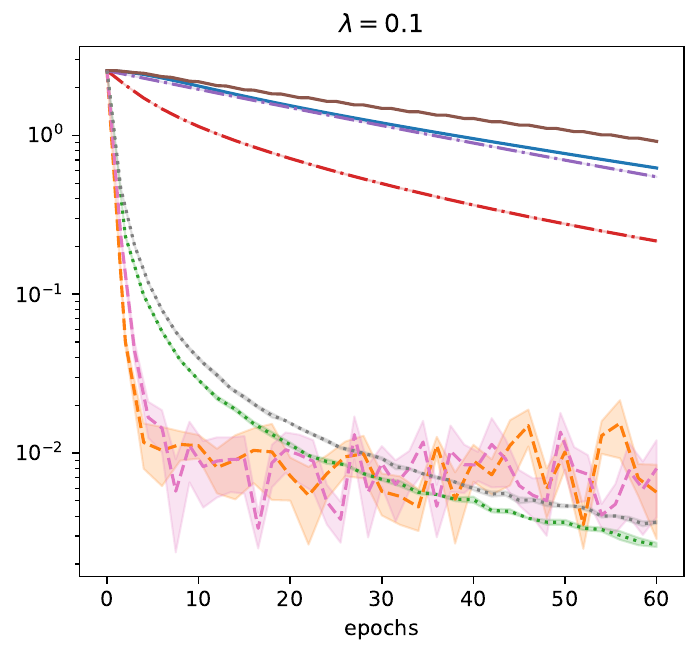}
    \hspace{-.25truecm}
    \includegraphics[width=0.45\textwidth]{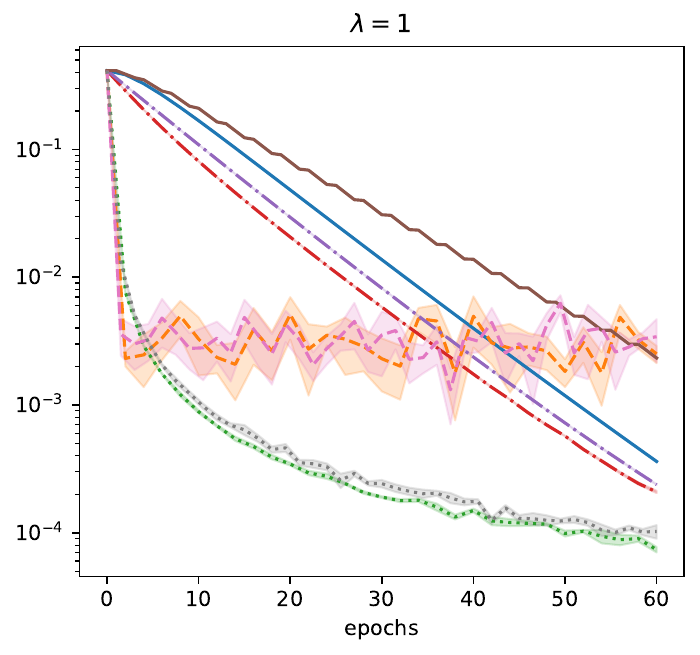}
    \hspace{-.25truecm}
    \includegraphics[width=0.45\textwidth]{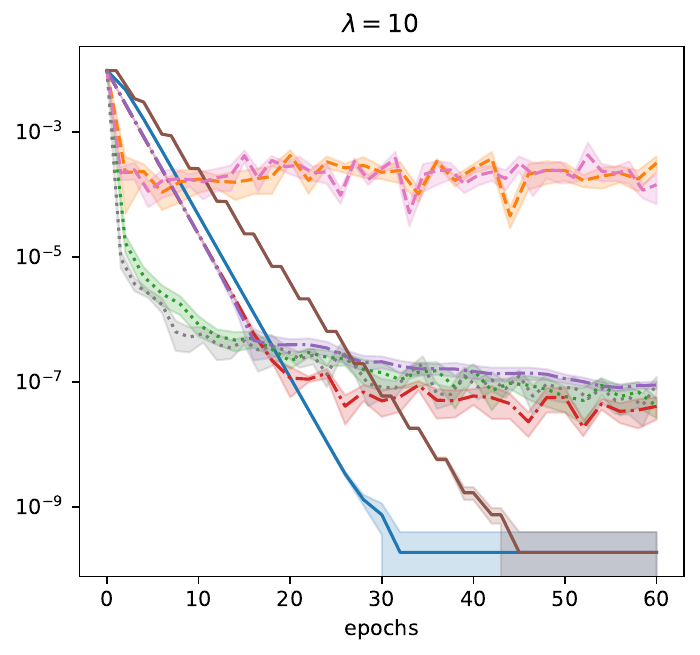}
    \includegraphics[width=0.45\textwidth]{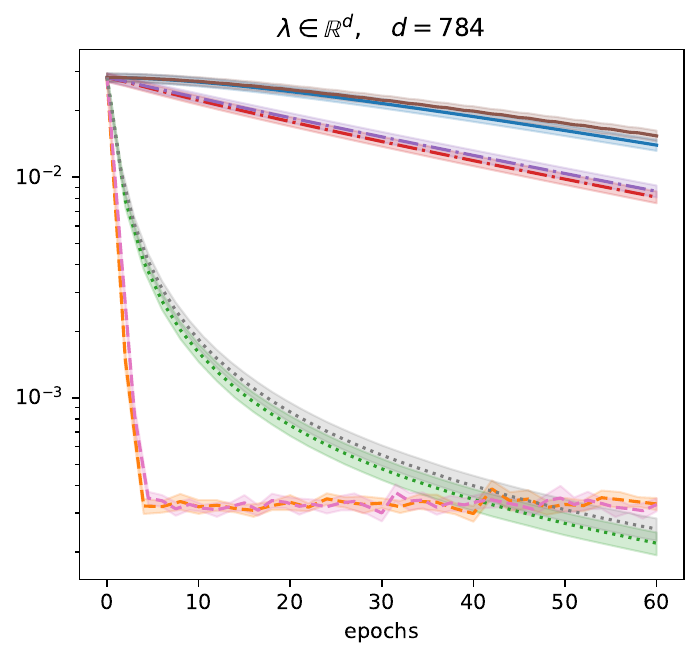}
    \caption{\small Experiments with a single (first 4 images) and multiple (last image) regularization parameters.
    The plots show mean (solid lines) and std (shaded regions) over 5 (first 4 images) and 10 (last image) runs. Each run varies the train/validation splits and, for the stochastic methods, the order and composition of the  minibatches. In addition, for each run in the last image, $\lambda_i = e^{\epsilon_i}$, where $\epsilon_i \sim \mathcal{U}[-2, 2]$ for every $i \in \{1,\dots, d\}$. 
    All methods use the same total computational budget. The first 
    five  use the same total number of epochs for solving the lower-level problem and the associated linear system. Whereas 
    the last three methods -- labeled with $75\%/25\%$ -- dedicate $3/4$ of epochs to solve the lower-level problem
    and only $1/4$ for the linear system. \small}
    \label{fig:onemulti}
\end{figure*}

\subsection{Bilevel Optimization on Twenty Newsgroup}

Here we replicate the setting of \citep{grazzi2020iteration} where multiple regularization parameters are optimized on the twenty newsgroup dataset. In particular, the lower-level objective is the $\ell_2$ regularized cross-entropy loss with one regularization parameter per feature computed on the training set, while the upper-level objective is the unregularized cross-entropy loss computed on the validation set. 

Differently from the previous experiments, which focused only on hypergradients, in this case we address the problem of minimizing the upper-level objective $f(\lambda)$. 
To minimize $f(\lambda)$ we use the \textrm{SGD} optimizer provided by \textrm{PyTorch} setting the learning rate to $10^3$. The approximate hypergradient is provided by one of the methods in \Cref{tb:methods} with a total budget of 20 epochs, meaning that each method exploits approximately 20 times the number of examples in the training set to compute the hypergradient. Following \citep{grazzi2020iteration}, we also warm-start the lower-level problem with the solution found at the previous upper-level iteration, which significantly improves the performance.
We note that each method starts by computing an approximation of $w(\lambda_0)$ which may provide different values of the considered metrics, even at the beginning of the procedure (see \Cref{fig:news}).

We halve the original training set to generate the training and validation sets, i.e. $\nt=\nv = 5657$, and we use minibatches of dimension 50 for the stochastic variants. We use the provided test set to compute the test performance metrics.

The performances varying the number of upper-level iterations are shown in \Cref{fig:news}. We can see that the pure stochastic variants outperform both the Batch and mixed methods. Furthermore, using the same number of epochs to solve the lower-level problem and the linear system appears to be the best strategy.
In \Cref{tb:news} we present the performance of the three main methods after completion of the bilevel optimization procedure.

\begin{figure*}[th]
    \centering
    \includegraphics[width=0.45\textwidth]{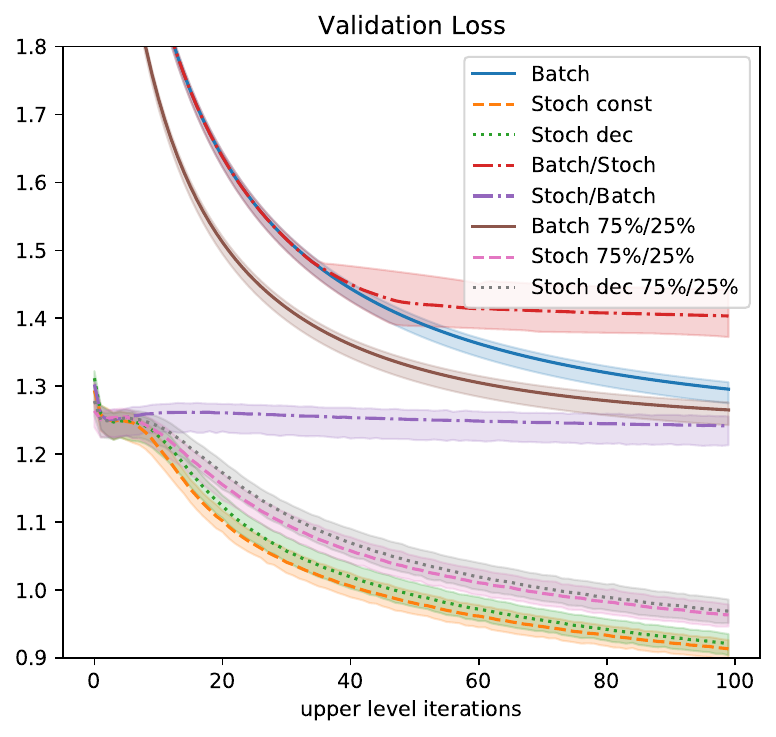}
    \includegraphics[width=0.45\textwidth]{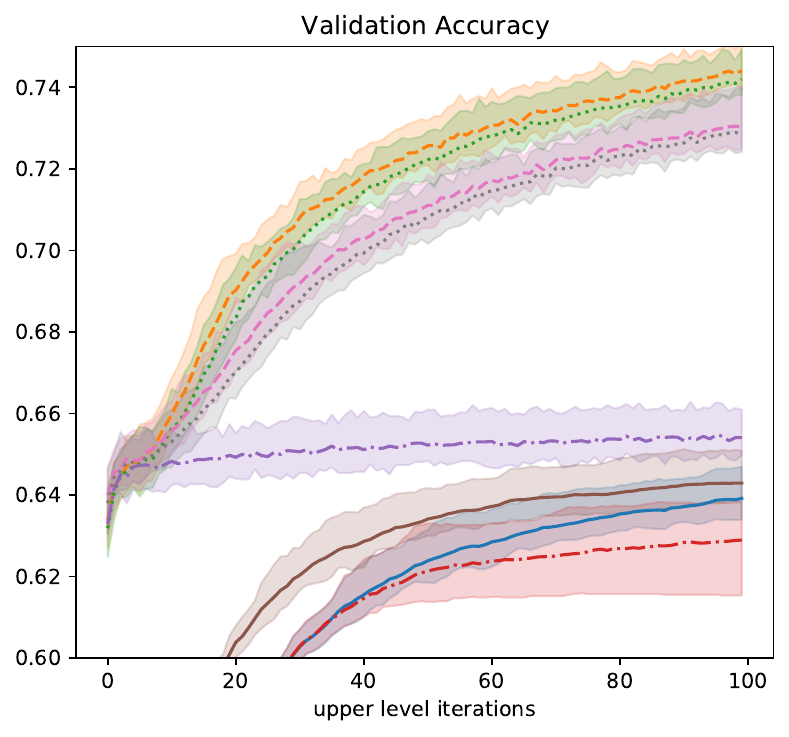}
    \includegraphics[width=0.45\textwidth]{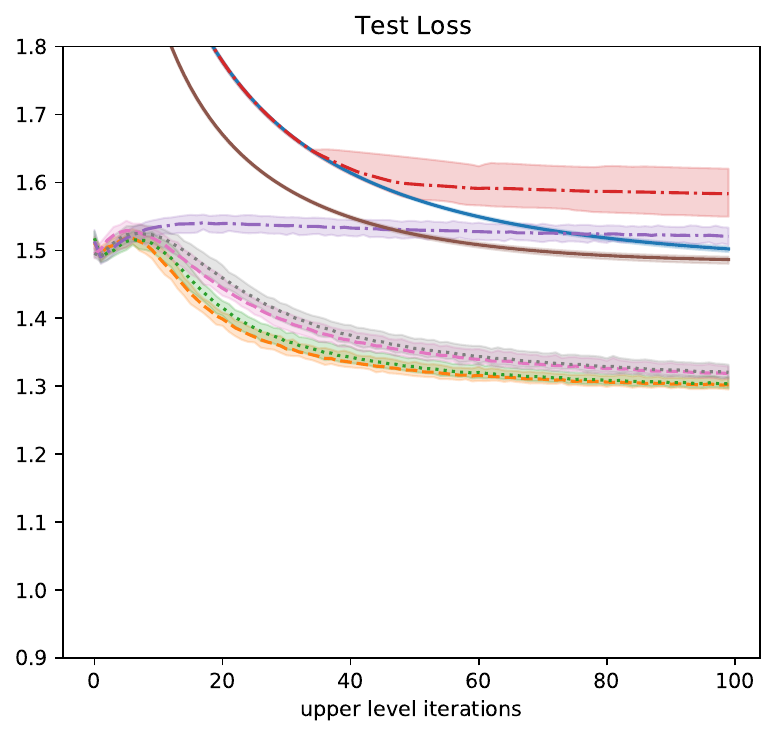}
    \includegraphics[width=0.45\textwidth]{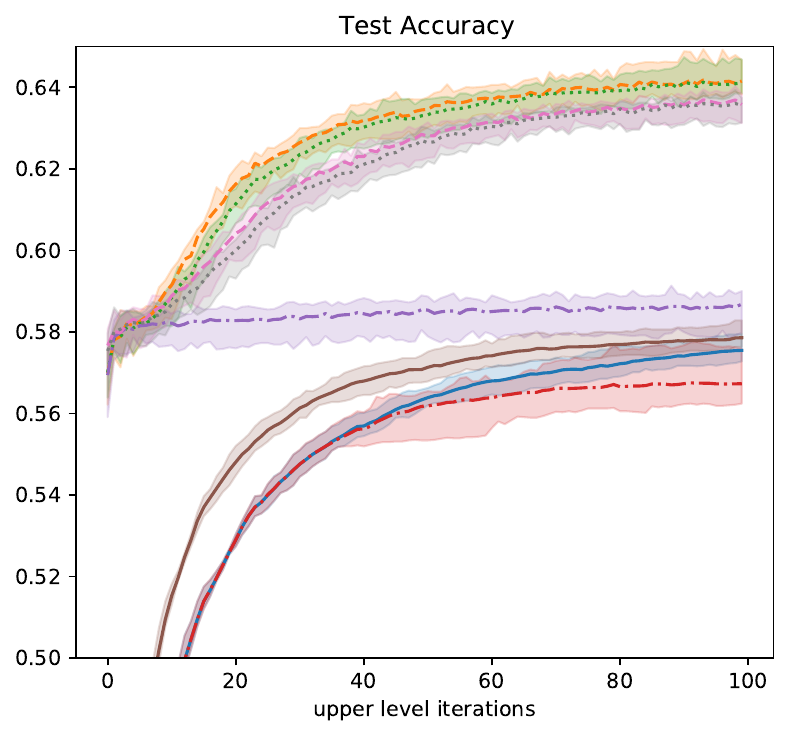}
    \caption{Performance metrics for multinomial logistic regression on twenty newsgroup.  All methods compute the hypergradient in $20$ epochs: methods labeled as $75\%/25\%$ compute the lower-level solution in $15$ epochs and the solution for the linear system in $5$, while the others solve both problems in $10$ epochs. The plots show  mean (solid line) and max-min (shaded region) over 5 runs varying both the train validation split and the mini-batch sampling of the stochastic algorithms. The starting point is the same for all methods and is set to $\lambda_0= 0$ as in \cite{grazzi2020iteration}.
    }
    \label{fig:news}
\end{figure*}

\begin{table*}[h]
\centering
\caption{Final performance metrics on the twenty newsgroup dataset, averaged over 5 trials. The metrics for the first three rows are obtained after 100 iterations of SGD on the upper-level objective.
The last row is the result for the conjugate gradient method obtained in (\cite{grazzi2020iteration}, Table~2) where they select the best  upper-level learning rate and perform 500 upper-level iterations. 
}\label{tb:news}
\vspace{+1mm}
\begin{tabular}[t]{lc|cc}
\toprule
\textbf{Method} & upper-level iter. & val. loss &   test acc. (\%) \\
\midrule
 \textit{Batch} $k=t=10$ & $100$ & 1.30 & 57.5 \\
\textit{Stoch dec.} $k=t=1131$ & $100$ & 0.92 & 64.1 \\
 \textit{Stoch const.} $k=t=1131$ & $100$ & 0.91 & 64.1 \\
 \midrule
 \textit{Batch} $k=t=10$ & $500$ & 0.93 & 63.7 \\
\bottomrule
\end{tabular}
\end{table*}

\vfill

\end{document}